%% file: paper-ICML-new.tex
\documentclass{article} 

\usepackage{tikz}

\newdimen\arrowsize
\pgfarrowsdeclare{arcsq}{arcsq}
{Co
  \arrowsize=0.2pt
  \advance\arrowsize by .5\pgflinewidth
  \pgfarrowsleftextend{-4\arrowsize-.5\pgflinewidth}
  \pgfarrowsrightextend{.5\pgflinewidth}
}
{
  \arrowsize=1.8pt
  \advance\arrowsize by .5\pgflinewidth
  \pgfsetdash{}{0pt} 
  \pgfsetroundjoin   
  \pgfsetroundcap    
  \pgfpathmoveto{\pgfpoint{0\arrowsize}{0\arrowsize}}
  \pgfpatharc{-90}{-140}{4\arrowsize}
  \pgfusepathqstroke
  \pgfpathmoveto{\pgfpointorigin}
  \pgfpatharc{90}{140}{4\arrowsize}
  \pgfusepathqstroke
}

\usepackage{times}
\usepackage{graphicx} 
\usepackage{caption} 
\usepackage{subcaption} 

\usepackage{natbib}

\usepackage{algorithm}
\usepackage{algorithmic}

\usepackage{hyperref}

\newcommand\Jonas[1]{{\color{red}Jonas: #1}}

\def\conferenceversion{1}
\def\journalversion{2}
\def\version{\journalversion}        
\def\iflong{\ifnum\version=\journalversion}
\def\ificml{\ifnum\version=\conferenceversion}

\ificml
\usepackage[accepted]{icml2015}
\fi

\iflong
\usepackage[accepted]{icml2015arxiv}
\fi

\input{definitions.tex}

\icmltitlerunning{Removing systematic errors for exoplanet search via latent causes}

\graphicspath{{./}}

\begin{document}

\renewcommand{\topfraction}{0.9}	
    \renewcommand{\bottomfraction}{0.9}	


\twocolumn[
\icmltitle{Removing systematic errors for exoplanet search via latent causes}

\icmlauthor{Bernhard Sch\"olkopf}{bs@tuebingen.mpg.de}
\icmladdress{Max Planck Institute for Intelligent Systems,
            72076 T\"ubingen, GERMANY}
\icmlauthor{David W. Hogg}{david.hogg@nyu.edu}
\icmlauthor{Dun Wang}{dw1519@nyu.edu}
\icmlauthor{Daniel Foreman-Mackey}{foreman.mackey@gmail.com}
\icmladdress{Center for Cosmology and Particle Physics, New York University, New York, NY 10003, USA}
\icmlauthor{Dominik Janzing}{dominik.janzing@tuebingen.mpg.de}
\icmlauthor{Carl-Johann Simon-Gabriel}{carl-johann.simon-gabriel@tuebingen.mpg.de}
\icmlauthor{Jonas Peters}{jonas.peters@tuebingen.mpg.de}
\icmladdress{Max Planck Institute for Intelligent Systems,
            72076 T\"ubingen, GERMANY}

\icmlkeywords{machine learning, ICML, causal inference, confounder, systematics, systematic error, exoplanet, astronomy, regression}

\vskip 0.3in
]

\begin{abstract}
  We describe a method for removing the effect of confounders in order
  to reconstruct a latent quantity of interest.  The method, referred to as {\em half-sibling regression}, is
  inspired by recent work in causal inference using additive noise
  models. We provide a theoretical justification and illustrate the
  potential of the method in a challenging astronomy application.
\end{abstract}

\section{Introduction}

The present paper proposes and analyzes a method for removing the effect of confounding noise. The analysis is based on a hypothetical underlying causal structure. The method does not infer causal structures; rather, it is influenced by a recent thrust to try to understand how causal structures facilitate machine learning tasks \cite{ScholkopfJPSZMJ2012}. 


Causal graphical models as pioneered by \citet{Pearl00,SpiGlySch93} are joint probability distributions over a set of variables $X_1,\dots,X_n$, along with directed graphs (usually, acyclicity is assumed) with vertices $X_i$, and arrows indicating direct causal influences. By the \emph{causal Markov assumption}, each vertex $X_i$ is independent of its non-descendants, given its parents. 

There is an alternative view of causal models, which does not start from a joint distribution. Instead, it assumes a set of jointly independent noise variables, one for each vertex, and a ``structural equation'' for each variable that describes how the latter is computed by evaluating a deterministic function of its noise variable and its parents. This view, referred to as a functional causal model (or nonlinear structural equation model), leads to the same class of joint distributions over all variables \cite{Pearl00,Peters2014anm}, and we may thus choose either representation.

The functional point of view is useful in that it often makes it easier to come up with assumptions on the causal mechanisms that are at work, i.e., on the functions associated with the variables. For instance, it was recently shown \cite{HoyJanMooPetetal09} that assuming nonlinear functions with additive noise renders the two--variable case identifiable --- i.e., a case where conditional independence tests do not provide any information, and it was thus previously believed that it is impossible to infer the structure of the graph based on observational data.

In this work we start from the functional point of view and assume the underlying causal graph shown in Fig.~\ref{fig:DAG}.
Here, $N,Q,X,Y$ are jointly random variables (RVs) (i.e., RVs defined on the same underlying probability space), taking values denoted by $n,q,x,y$.
We do not require the ranges of the random variables to be $\RR$, in particular, they may be vectorial. All equalities regarding random variables should be interpreted to hold with probability one. We further (implicitly) assume the existence of conditional expectations.
\iflong

Note that while the causal motivation was helpful for our work, one can also view Fig.~\ref{fig:DAG} as a DAG (directed acyclic graph) without causal interpretation, i.e., as a directed a graphical model. We need $Q$ and $X$ (and in some cases also $N$) to be independent, which follows from the given structure no matter whether one views this as a causal graph or as a graphical model. 

In the next section, we present the method. Section~\ref{sec:exp} describes the application and provides experimental results, and Section~\ref{sec:conc} summarizes our conclusions.
\fi

\section{Half-Sibling Regression}
\begin{figure}[bt]
\begin{center}
\begin{tikzpicture}[xscale = 2, yscale=1.5, line width=0.5pt, inner sep=0.5mm, shorten >=1pt, shorten <=1pt]
  \draw (-.5,2) node(oo) [] {unobserved};
  \draw (-.5,1) node(o) [] {observed};
  \draw (1,1) node(y) [circle, draw] {$Y$};
  \draw (2.6,1) node(x) [circle, draw] {$X$};
  \draw (1.8,2) node(n) [circle, draw] {$N$};
  \draw (1,2) node(q) [circle, draw] {$Q$};
  \draw[-arcsq] (q) -- (y);
  \draw[-arcsq] (n) -- (y);
  \draw[-arcsq] (n) -- (x);
\end{tikzpicture}
\end{center}
\caption{
We are interested in reconstructing the quantity $Q$ based on the observables $X$ and $Y$ affected by noise $N$, using the knowledge that $(N,X) \independent Q$. Note that the involved quantities need not be scalars, which makes the model more general than it seems at first glance. For instance, we can think of $N$ as a multi-dimensional vector, some components of which affect only $X$, some only $Y$, and some both $X$ and $Y$. 
\label{fig:DAG}}
\end{figure}
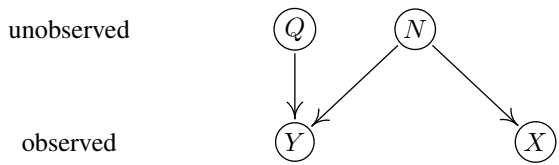

Suppose we are interested in the quantity $Q$, but unfortunately we cannot observe it directly. Instead, we observe $Y$, which we think of as a degraded version of $Q$ that is affected by noise $N$. Clearly, without knowledge of $N$, there is no way to recover $Q$. However, we assume that $N$ also affects another observable quantity (or a collection of quantities) $X$. By the graph structure, conditional on $Y$, the variables $Q$ and $X$ are dependent (in the generic case), thus $X$ contains information about $Q$. This situation is quite common if $X$ and $Y$ are measurements performed with the same apparatus, introducing the noise $N$. In the physical sciences, this is often referred to as {\em systematics}, to convey the intuition that these errors are not simply due to random fluctuations, but caused by systematic influences of the measuring device. In our application below, both types of errors occur, but we will not try to tease them apart.
\iflong
Our method addresses errors that affect both $X$ and $Y$, for instance by acting on $N$, no matter whether we call them random or systematic.
\fi

How can we use this information in practice? Unfortunately, without further restrictions, this problem is still too hard. 
Suppose that $N$ randomly switches between $\{ 1,\dots,v\}$, where $v\in\NN$ \cite{ScholkopfJPSZMJ2012}. Define the structural equation $f_Y$ for the variable $Y$ as follows: $y= f_Y(n,q) := f_n(q)$, where $f_1,\dots,f_v$ are $v$ distinct functions that compute $Y$ from $Q$ --- in other words, we randomly switch between $v$ different mechanisms. Clearly, no matter how many pairs $(x,y)$ we observe, we can choose a sufficiently large $v$ along with functions $f_1,\dots,f_v$ such that there is no way of gleaning any reliable information on $Q$ from the $f_i(Q)$ --- e.g., there may be more $f_i$ than there were data points. Things could get even worse: for instance, $N$ could be real valued, and switch between an uncountable number of functions. To prevent this kind of behavior, we need to simplify the way in which $Y$ is allowed to depend on $N$.



Before we do so, we need to point out a fundamental limitation. 
The above example shows that it can be arbitrarily hard to get information about $Q$ from finite data. However, even from infinite data, only partial information is available and certain ``gauge'' degrees of freedom remain.\iflong\footnote{This means that there are some degrees of freedom in the parametrization of the model which do not affect the observable model. } \fi
~In particular, given a reconstructed $Q$, we can always construct another one by applying an invertible transformation to it, and incorporating its inverse into the function computing $Y$ from $Q$ and $N$. This includes the possibility of adding an offset, which we will see below.

We next propose an assumption which allows for a practical method to solve the problem of reconstructing $Q$ up to the above gauge freedom. The method is surprisingly simple, and while we have not seen it in the same form elsewhere, we do not want to claim originality for it. Related tricks are occasionally applied in practice, often employing factor analysis to account for confounding effects \cite{Price06,Yu06,JohLi07,Kang08,Stegle08,GagSpe11}.
We will also present a theoretical analysis that provides insight into why and when these methods work.

\subsection{Complete Information} \label{sec:complete}
Inspired by recent work in causal inference, we use nonlinear additive noise models \cite{HoyJanMooPetetal09}. Specifically, we assume that there exists a function $f$ such that
\begin{equation}\label{eq:additive}
Y = Q + f(N).
\end{equation}
Note that we could equally well assume the more general form $Y = g(Q) + f(N)$, and the following analysis would look the same. However, in view of the above remark about the gauge freedom, this is not necessary since $Q$ can at most be identified up to a (nonlinear) reparametrization anyway. Note, moreover, that while for \citet{HoyJanMooPetetal09}, the input of $f$ is observed and we want to decide if it is a cause of $Y$, in the present setting the input of $f$ is unobserved \cite{JanPetMooSch09}, and the goal is to recover $Q$, which for \citet{HoyJanMooPetetal09} played the role of the noise.

The intuition behind our approach is as follows. Since $X\independent Q$, $X$ cannot predict $Q$, and thus neither $Q$'s influence on $Y$. It may contain information, however, about the influence of $N$ on $Y$, since $X$ is also influenced by $N$. Now suppose we try to predict $Y$ from $X$. As argued above, whatever comes from $Q$ cannot be predicted, hence only the component coming from $N$ will be picked up. Trying to predict $Y$ from $X$ is thus a vehicle to selectively capture $N$'s influence on $Y$, with the goal of subsequently removing it, to obtain an estimate of $Q$ referred to as $\hat{Q}$:
\begin{definition}\label{def:Q}
\begin{equation}\label{eq:Q}
\hat{Q}:=Y-E[Y|X]
\end{equation}
\end{definition}

For an additive model (\ref{eq:additive}), our intuition can be formalized: in this case, we can predict the additive component in $Y$ coming from $N$ --- which is exactly what we want to remove to cancel the confounding effect of $N$ and thus reconstruct $Q$ (up to an offset):
\begin{proposition}\label{prop}
Suppose $N,X$ are jointly random variables, and $f$ is a measurable function. If there exists a function $\psi$ such that 
\begin{equation}\label{eq:psi}
f(N)=\psi(X), 
\end{equation}
i.e., $f(N)$ can in principle be predicted from $X$ perfectly, then we have
\begin{equation}\label{eq:prop-a}
f(N) = E[f(N)|X].
\end{equation}
If, moreover, the additive model assumption (\ref{eq:additive}) holds, with $Q,Y$ RVs on the same underlying probability space, and $Q\independent X$, then 
\begin{equation}\label{eq:prop-b}
\hat{Q} = Q - E[Q].
\end{equation}
\end{proposition}

In our main application below, $N$ will be systematic errors from an astronomical spacecraft and telescope, $Y$ will be a star under analysis, and $X$ will be a large set of other stars. In this case, the assumption that $f(N)=\psi(X)$ has a concrete interpretation: it means that the device can be {\em self-calibrated} based on measured science data only \cite{self-calib}.

\begin{proof}
Due to \eqref{eq:psi}, we have
\begin{equation}
E[f(N)|X] = E[\psi(X)|X] = \psi(X) = f(N).
\end{equation}

To show the second statement, consider the conditional expectation 
\begin{equation}
E[Y|X] = E[Q+f(N)|X] 
\end{equation}
Using $Q\independent X$ and \eqref{eq:prop-a}, we get
\begin{equation}
E[Y|X] =E[Q] +f(N)
= E[Q] + Y - Q.
\end{equation}
Recalling Definition~\ref{def:Q} completes the proof.
\end{proof}


Proposition~\ref{prop} 
provides us with a principled recommendation how to remove the effect of the noise and reconstruct the unobserved $Q$ up to its mean $E[Q]$: we need to subtract the conditional expectation (i.e., the regression) $E[Y|X]$ from the observed $Y$ (Definition~\ref{def:Q}). The regression $E[Y|X]$ can be estimated from observations $(x_i,y_i)$ using (linear or nonlinear) off-the-shelf methods. 
We refer to this procedure as {\em half-sibling regression} to reflect the fact that we are trying to explain aspects of the child $Y$ by regression on its half-sibling(s) $X$ in order to reconstruct properties of its unobserved parent $Q$.


Note that $m(x):=E[f(N)|X=x]$ is a function of $x$, and $E[f(N)|X]$ is the random variable $m(X)$. Correspondingly, (\ref{eq:prop-a}) is an equality of RVs. By assumption, all RVs live on the same underlying probability space. If we perform the associated random experiment, we obtain values for $X$ and $N$, and (\ref{eq:prop-a}) tells us that if we substitute them into $m$ and $f$, respectively, we get the same value with probability 1.
Eq.~(\ref{eq:prop-b}) is also an equality of RVs, and the above procedure therefore not only reconstructs some properties of the unobservable RV $Q$ --- it reconstructs, up to the mean $E[Q]$, and with probability 1, the RV itself. This may sound too good to be true --- in practice, of course its accuracy will depend on how well the assumptions of 
Proposition~\ref{prop} 
hold.

If the following conditions are met, we may expect that the procedure should work well in practice:

(i) $X$ should be (almost) independent of $Q$ --- otherwise, our method could possibly remove parts of $Q$ itself, and thus throw out the baby with the bathtub. A sufficient condition for this to be the case is that $N$ be (almost) independent of $Q$, which often makes sense in practice, e.g., if $N$ is introduced by a measuring device in a way independent of the underlying object being measured. Clearly, we can only hope to remove noise that is independent of the signal, otherwise it would be unclear what is noise and what is signal. A sufficient condition for $N\independent Q$, finally, is that the causal DAG in Fig.~\ref{fig:DAG} correctly describes the underlying causal structure.
\iflong
~\\
\fi
Note, however, that 
Proposition~\ref{prop} 
and thus our method also applies if $N\not\independent Q$, as long as $X\independent Q$. 

(ii)
The observable $X$ is chosen such that $Y$ can be predicted as well as possible from it; i.e., $X$ contains enough information about $f(N)$ and, ideally, $N$ acts on both $X$ and $Y$ in similar ways such that a ``simple'' function class suffices for solving the regression problem in practice.\\
This may sound like a rather strong requirement, but we will see that in our astronomy application, it is not unrealistic: $X$ will be a large vector of pixels of other stars, and we will use them to predict a pixel $Y$ of a star of interest. In this kind of problem, the main variability of $Y$ will often be due to the systematic effects due to the instrument $N$ also affecting other stars, and thus a large set of other stars will indeed allow a good prediction of the measured $Y$.

Note that it is not required that the underlying structural equation model be linear --- $N$ can act on $X$ and $Y$ in nonlinear ways, as an additive term $f(N)$.

In practice, we never observe $N$ directly, and thus it is hard to tell whether the assumption of perfect predictability of $f(N)$ from $X$ holds true. We now relax this assumption.

\subsection{Incomplete Information} \label{sec:incompl}


First we observe that $E[f(N)|X]$ is a good approximation for
$f(N)$ whenever $f(N)$ is almost determined by $X$:
\begin{lemma}\label{lem:condvar}
For any two jointly random variables $Z,X$, we have
\begin{equation}\label{eq:condvar}
E [(Z-E [Z|X])^2] =E [\Var [Z|X]] .
\end{equation}
\end{lemma}
Here, $E [Z|X]$ is the random variable $g(X)$ with $g(x)=E [Z|X=x]$, and
$\Var[Z|X]$ is the random variable $h(X)$ with
$h(x)=\Var[Z|X=x]$. Then (\ref{eq:condvar}) turns into 
\begin{equation}
E[(Z-g(X))^2]=E[h(X)]\,.
\end{equation}
\begin{proof}
Note that for any random variable $Z$ we have
\[
\Var [Z|X=x]=E [(Z-E [Z|X=x])^2|X=x]\,,
\]
by the definition of variance, applied to the
variable $Z|_{X=x}$. Hence
\[
\Var [Z|X]=E [(Z-E [Z|X])^2|X]\,,
\]
where both sides are functions of $X$.
Taking the expectation w.r.t.\ $X$ on both sides yields
\[
E [\Var [Z|X]]=E [(Z-E [Z|X])^2]\,,
\]
where we have used the law of total expectation 
$
E [E [W|X]]=E [W]
$
 on the right hand side. 
\end{proof}


This leads to a stronger result for our estimator $\hat{Q}$ \eqref{eq:Q}:
\begin{proposition}\label{lem:crucial}
Let $f$ be measurable, $N,Q,X,Y$ jointly random variables with $Q\independent X$, and $Y=Q+f(N)$.
The expected squared deviation between $\hat{Q}$  and $Q-E[Q]$ satisfies
\begin{equation}\label{eq:lemma4}
E[(\hat{Q} -(Q-E[Q]))^2]= E[\Var[f(N)|X]]\,.
\end{equation}
\end{proposition}
\begin{proof}
We rewrite the argument of the square in \eqref{eq:lemma4} as
\begin{align*}
& \hat{Q}-(Q-E[Q]) \\
=\;& Y-E[Y|X] -Q+E[Q]\\
=\;& f(N)+Q -E[f(N)|X]-E[Q|X] -Q+E[Q]\\
=\;& f(N) -E[f(N)|X].
\end{align*}
\iflong
Here, the last step uses $E[Q|X]=E[Q]$, which follows from $Q\independent X$.

\fi
The result follows using Lemma~\ref{lem:condvar} with $Z:=f(N)$.
\end{proof}

Note that 
Proposition~\ref{prop}
is a special case of 
Proposition~\ref{lem:crucial}:
if there exists a function $\psi$ such that $\psi(X)=f(N)$, then the r.h.s.\ of \eqref{eq:lemma4} vanishes.
Proposition~\ref{lem:crucial} 
drops this assumption, which is more realistic: consider the case where $X=g(N)+R$, where $R$ is another random variable. In this case, we cannot expect to reconstruct the variable $f(N)$ from $X$ exactly.

There are, however, two settings where we would still expect good approximate recovery of $Q$:

(i) If the standard deviation of $R$ goes to zero, the signal of $N$ in $X$ becomes strong and we can approximately estimate $f(N)$ from $X$, see 
Proposition~\ref{prop:decrsd}.

(ii) Alternatively, we observe many different effects of $N$. In the astronomy application below, $Q$ and $R$ are stars, from which we get noisy observations $Y$ and $X$. 
Proposition~\ref{prop:incrp} 
below shows that observing many different $X_i$ helps reconstructing $Q$, even if all $X_i$ depend on $N$ through different functions $g_i$ and their underlying (independent) signals $R_i$ do not follow the same distribution.
\iflong
The intuition is that with increasing number of variables the independent $R_i$ ``average'' out, and thus it becomes easier to reconstruct the effect of $N$. 
\fi

\begin{proposition} \label{prop:decrsd}
Assume that $Y = Q + f(N)$ and let
$$
X^s := g(N) + s \cdot R\,,
$$
where $R$, $N$ and $Q$ are jointly independent, $f \in C_b^1(\mathbb{R})$, $g \in C^1(\mathbb{R})$, $s\in\RR$, and $g$ is invertible.
Then
$$
\hat Q^s \overset{L^2}{\rightarrow} Q - E[Q] \, \quad \text{ as } \quad s \rightarrow 0\,,
$$
where
$
\hat Q^s := Y - E[Y | X^s]
$.
\end{proposition}

\ificml
The proof is skipped since it is analogous to the more complicated proof of Proposition~\ref{prop:incrp}.
\fi
\iflong
\begin{proof}
We have for $s \rightarrow 0$ that
\begin{eqnarray*}
& s\cdot R \! & \! \overset{P}{\rightarrow} 0 \\
\Rightarrow \! & \!
g(N) + s\cdot R - g(N) &\overset{P}{\rightarrow} 0\\
\overset{*}{\Rightarrow}  \! & \!
g^{-1} \left( g(N) + s\cdot R \right) - N \! & \! \overset{P}{\rightarrow} 0\\
\overset{*}{\Rightarrow} \! & \!
f\left(
g^{-1} \left( g(N) + s\cdot R \right)\right) - f(N)\! & \! \overset{P}{\rightarrow} 0\\
\Rightarrow  \! & \! \psi_s (X^s) - f(N) \! & \! \overset{P}{\rightarrow} 0
\end{eqnarray*}
for some $\psi_s$ that is bounded in $s$ (the implications $*$ follow from the continuous mapping theorem).\footnote{The notation $\overset{P}{\rightarrow}$ denotes convergence in probability with respect to the measure $P$ of the underlying probability space.} This implies
$$
E[f(N) | X^s] - f(N)  \overset{L^2}{\rightarrow} 0
$$
because
$$
E[\left(f(N) - E[f(N)| X^s]\right)^2] \leq E[\left(f(N) - \psi_s(X^s) \right)^2] {\rightarrow} 0
$$
($L^2$ convergence follows because $f$ is bounded).
But then
\begin{align*}
Q - E[Q] - \hat Q^s
&= - f(N) -E[Q] + E[f(N) + Q | X^s] \\
&= E[f(N) | X^s] - f(N)  \overset{L^2}{\rightarrow} 0
\end{align*}
\end{proof}
\fi
\input{proposition5}



{
}

\ificml
The next subsection mentions an optional extension of our approach. Another extension, to the case of time-dependent data, is briefly discussed elsewhere \cite{Schoelkopf15-arxiv}. 
\fi

\iflong
The next two subsections discuss optional extensions of our approach. Readers who are mainly interested in the application may prefer to move to  Section~\ref{sec:exp} directly.

\subsection{Time Series}
\input{timeseries.tex}

\fi

\subsection{Prediction from Non-Effects of the Noise Variable}
While Fig.~\ref{fig:DAG} shows the causal structure motivating our work, our method does not require a directed arrow from $N$ to $X$ --- it only requires that $N\not\independent X$, to ensure that $X$ contains information about $N$. We can represent this by an undirected connection between the two (Fig.~\ref{fig:DAG-undirected}), and note that such a dependence may arise from an arrow directed in either direction, and/or another confounder that influences both $N$ and $X$. This confounder need not act deterministically on $N$, hence effectively removing our earlier requirement of a deterministic effect, cf.\ \eqref{eq:additive}.

\begin{figure}[tb]
\begin{center}
\begin{tikzpicture}[xscale = 2, yscale=1.3, line width=0.5pt, inner sep=0.5mm, shorten >=1pt, shorten <=1pt]
  \draw (-.5,2) node(oo) [] {unobserved};
  \draw (-.5,1) node(o) [] {observed};
  \draw (1,1) node(y) [circle, draw] {$Y$};
  \draw (2.6,1) node(x) [circle, draw] {$X$};
  \draw (1.8,2) node(n) [circle, draw] {$N$};
  \draw (1,2) node(q) [circle, draw] {$Q$};
  \draw[-arcsq] (q) -- (y);
  \draw[-arcsq] (n) -- (y);
  \draw[-] (n) -- (x);
\end{tikzpicture}
\end{center}
\caption{Causal structure from Fig.~\ref{fig:DAG} when relaxing the assumption that $X$ is an effect of $N$.\label{fig:DAG-undirected}}
\end{figure}

\section{Applications\label{sec:exp}}
\subsection{Synthetic Data}

\input{experiments1.tex}

\iflong
\begin{figure}[tbh]
\begin{center}
 {\includegraphics[width=1\columnwidth]{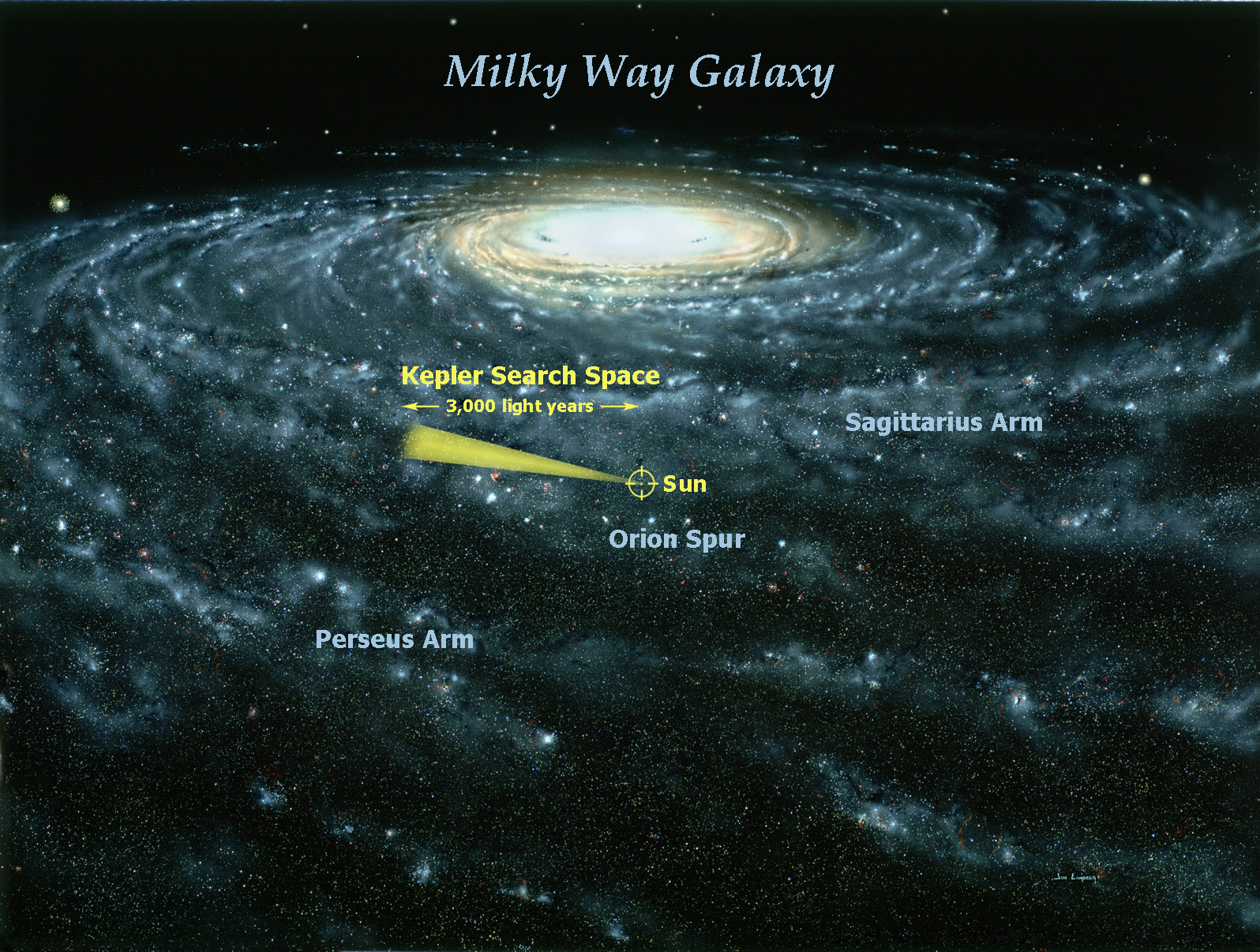}}
\end{center}
\caption{\label{fig1a}View of the Milky Way with position of the sun and depiction of the Kepler search field (image credit: NASA).}
\end{figure}
\fi

\subsection{Exoplanet Light Curves}
The field of exoplanet search has recently become one of the most popular areas of astronomy research. This is largely due to the Kepler space observatory launched in 2009. Kepler observed a tiny fraction of the Milky Way in search of exoplanets. The telescope was pointed at same patch of sky for more than four years 
\iflong
(Fig.~\ref{fig1a} and \ref{fig1b}). 
\fi
\ificml
(Fig.~\ref{fig1b}). 
\fi
In that patch, it monitored the brightness of 150000 stars (selected from among 3.5 million stars in the search field), taking a stream of half-hour exposures using a set of CCD (Charge-Coupled Device) imaging chips arranged in its focal plane using the layout visible in Fig.~\ref{fig1b}.

\begin{figure}[tbh]
\begin{center}
 {\includegraphics[width=\columnwidth]{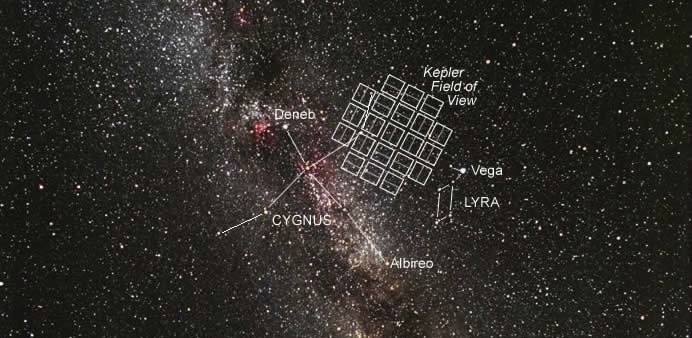}}
\end{center}
\caption{\label{fig1b}Kepler search field as seen from Earth, located close to the Milky Way plane, in a star-rich area near the constellation Cygnus (image credit: NASA).}
\end{figure}

\begin{figure}[bth]
\begin{center}
 {\includegraphics[width=1\columnwidth]{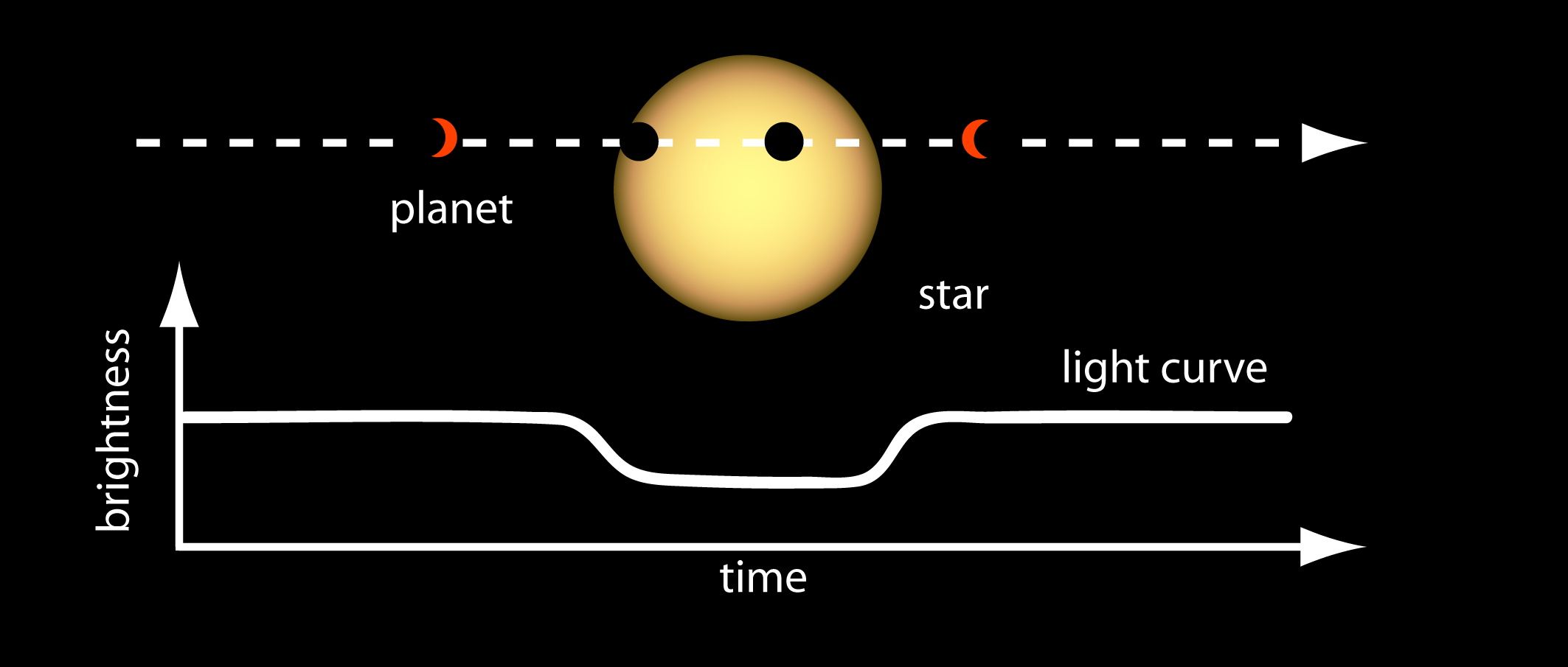}}
\end{center}
\caption{\label{transit}Sketch of the transit method for exoplanet detection. As a planet passes in front of its host star, we can observe a small dip in the apparent star brightness (image credit: NASA Ames).}
\end{figure}

Kepler detects exoplanets using the {\em transit method.} Whenever a planet passes in front of their host star(s), we observe a tiny dip in the light curve (Fig.~\ref{transit}). This signal is rather faint, and for our own planet as seen from space, it would amount to a brightness change smaller than $10^{-4}$, lasting less than half a day, taking place once a year, and visible from about half a percent of all directions. The level of required photometric precision to detect such transits is one of the main motivations for performing these observations in space, where they are not disturbed by atmospheric effects, and it is possible to observe the same patch almost continuously using the same instrument. 

For planets orbiting stars in the habitable zone (allowing for liquid water) of stars similar to the sun, we would expect the signal to be observable at most every few months. We thus have very few observations of each transit. However, it has become clear that there is a number of confounders introduced by spacecraft and telescope that lead to systematic changes in the light curves which are of the same magnitude or larger than the required accuracy. The dominant error is pointing jitter: if the camera field moves by a tiny fraction of a pixel (for Kepler, the order of magnitude is 0.01 pixels), then the light distribution on the pixels will change. 
Each star affects a set of pixels
\iflong
(Fig.~\ref{ccd}),
\fi
\ificml
(Fig.~xxxy in \cite{Schoelkopf15-arxiv}),
\fi
and we integrate their measurements to get an estimate of the star's overall brightness. Unfortunately, the pixel sensitivities are not precisely identical, and even though one can try to correct for this, we are left with significant systematic errors. Overall, although Kepler is highly optimized for stable photometric measurements, its accuracy falls short of what is required for reliably detecting earth-like planets in habitable zones of sun-like stars.

\iflong
\begin{figure*}[p!]
\centering
\begin{subfigure}[htb]{1.6\columnwidth}
\includegraphics[width=\columnwidth]{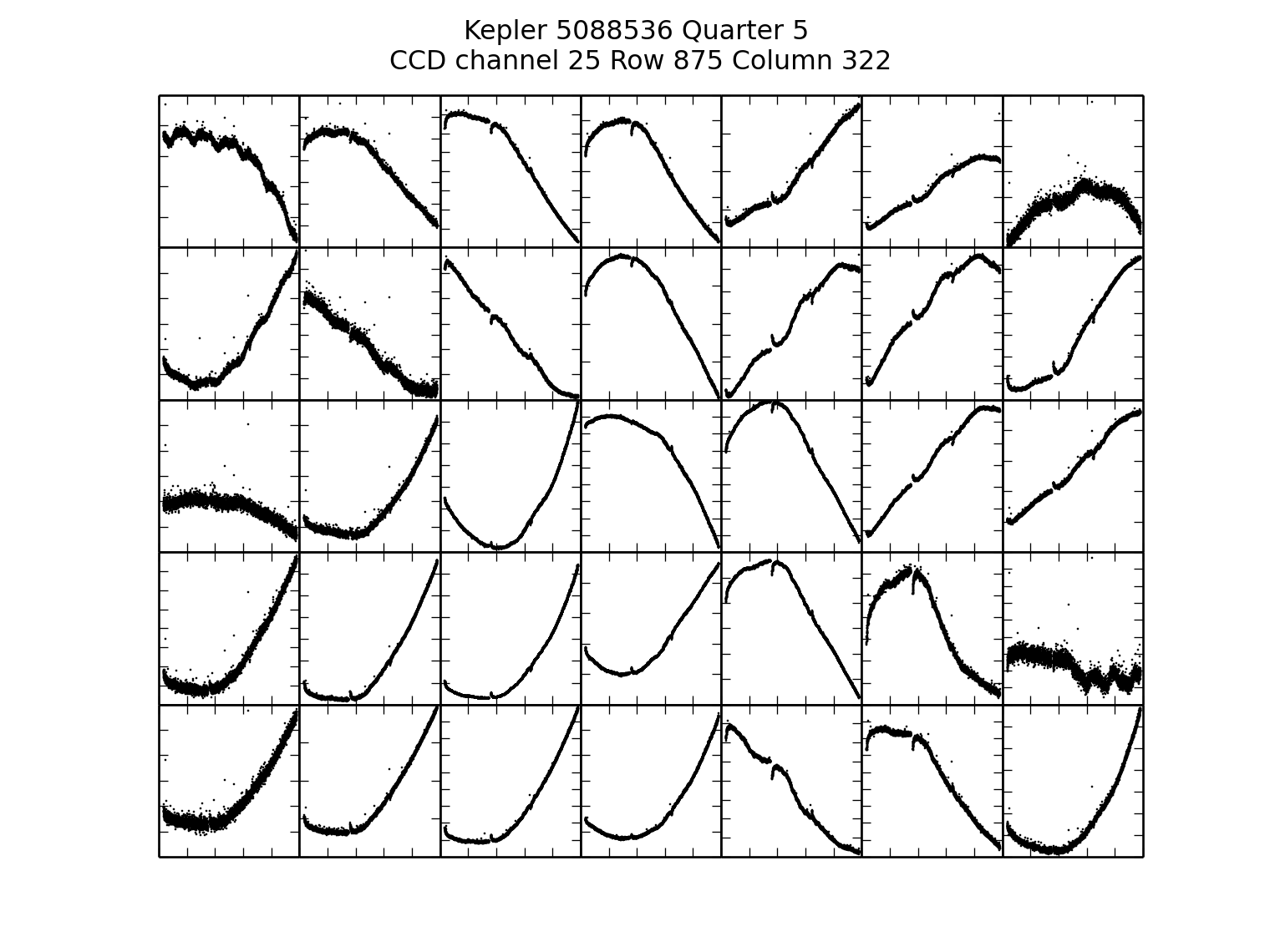}
\caption{}
\end{subfigure}

\begin{subfigure}[htb]{1.6\columnwidth}
\includegraphics[width=\columnwidth]{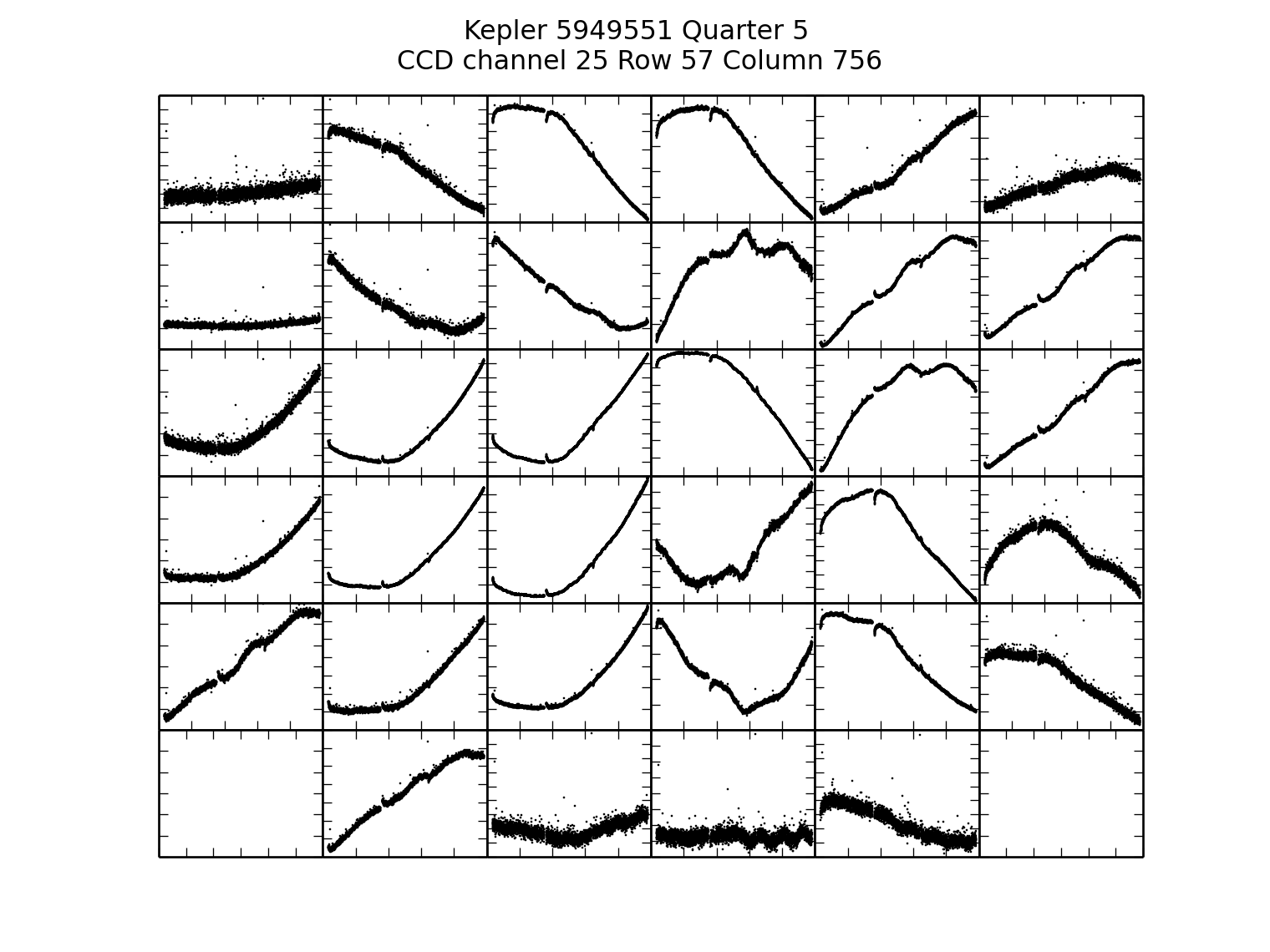}
\caption{}
\end{subfigure}
\caption{\label{ccd} Stars on the same CCD share systematic errors. The two panels show pixel fluxes (brightnesses) for two stars: (a) KIC 5088536, (b) KIC 5949551; here, KIC stands for Kepler Input Catalog. Both stars lie on the same CCD, but far enough apart such that there is no stray light from one affecting the other. Each panel shows the pixels contributing to the respective star. Note that there exist similar trends in some pixels of these two stars, caused by systematic errors. 
}
\end{figure*}
\fi

We obtained the data from the Mikulski Archive for Space Telescopes (MAST) (see \url{http://archive.stsci.edu/index.html}).
Our system, which we abbreviate as CPM (Causal Pixel Model), 
is based on the assumption that stars on the same CCD share systematic errors. If we pick two stars on the same CCD that are far away from each other, they will be light years apart in space and no physical interaction between them can take place. 
\iflong
As Fig.~\ref{ccd} shows, 
\fi
\ificml
As Fig.~xxxy in \cite{Schoelkopf15-arxiv} shows, 
\fi
the light curves nevertheless have similar trends, which is caused by systematics. In CPM, we use linear regression to predict the light curve of each pixel belonging to the target star as a linear combination of a set of predictor pixels. Specifically, we use 4000 predictor pixels from about 150 stars, which are selected to be closest in magnitude to the target star.\footnote{The exact number of stars varies with brightness, as brighter stars have larger images on the CCD and thus more pixels.} This is done since the systematic effects of the instruments depend somewhat on the star brightness; e.g., when a star saturates a pixel, blooming takes place and the signal leaks to neighboring pixels. To rule out any direct optical cross-talk by stray light, we require that the predictor pixels are from stars sufficiently far away from the target star (at least 20 pixels distance on the CCD), but we always take them from the same CCD (note that Kepler has a number of CCDs, and we expect that systematic errors depend on the CCD).  We train the model separately for each month, which contains about 1300 data points.\footnote{The data come in batches which are separated by larger errors, since the spacecraft needs to periodically re-direct its antenna to send the data back to earth.}
 Standard  L2 regularization is employed to avoid overfitting, and parameters (regularization strength and number of input pixels) were optimized using cross-validation. Nonlinear kernel regression was also evaluated, but did not lead to better results. This may be due to the fact that the set of predictor pixels is relatively large (compared to the training set size); and among this large set, it seems that there are sufficiently many pixels who are affected by the systematics in a rather similar way as the target.

We have observed in our results that the method removes some of the intrinsic variability of the target star. This is due to the fact that the signals are not i.i.d.\ and time acts as a confounder. If among the predictor stars, there exists one whose intrinsic variability is very similar to the target star, then the regression can attenuate variability in the latter. This is unlikely to work exactly, but given the limited observation window, an approximate match (e.g., stars varying at slightly different frequencies) will already lead to some amount of attenuation. Since exoplanet transits are very rare, it is extremely unlikely (but not impossible) that the same mechanism will remove some transits.

Note that for the purpose of exoplanet search, the stellar variability can be considered a confounder as well, 
independent of the planet positions which are causal for transits. In order to remove this, we use as additional regression inputs also past and future of the target star. This adds an autoregressive (AR) component to our model, removing more of the stellar variability and thus increasing the sensitivity for transits. In this case, we select an exclusion window around the point of time being corrected, to ensure that we do not remove the transit itself. Below, we report results where the AR component uses as inputs the three closest future and the three closest past time points, subject to the constraint that a window of $\pm$9 hours around the considered time point is excluded. Choosing this window corresponds to the assumption that time points earlier than -9 hours or later than +9 hours are not informative for the transit itself. Smaller windows allow more accurate prediction, at the risk of damaging slow transit signals.
Our code is available at \url{https://github.com/jvc2688/KeplerPixelModel}.

To give a view on how our method performs, CPM is applied on several stars with known transit signals. After that, we compare them with the Kepler Pre-search Data Conditioning (PDC) method (see \url{http://keplergo.arc.nasa.gov/PipelinePDC.shtml}). 
\iflong
PDC builds on the idea that systematic errors have a temporal structure that can be extracted from ancillary quantities. 
\fi 
The first version of PDC removed systematic errors based on correlations with a set of ancillary engineering data, including temperatures at the detector electronics below the CCD array, 
  and polynomials describing centroid motions of stars. The current PDC 
\ificml
\cite{pdc2}
\fi
\iflong
\cite{pdc2,pdc3}
\fi
performs PCA on filtered light curves of stars, projects the light curve of the target star on a PCA subspace, and subsequently removes this projection. 
The PCA is performed on a set of relatively quiet stars close in position and magnitude. 
For non-i.i.d.\ data, this procedure could remove temporal structure of interest. To prevent this, the PCA subspace is restricted to eight dimensions, strongly limiting the capacity of the model \citep[cf.][]{Foreman-Mackey15}.

\begin{figure*}
\centering
\begin{subfigure}[htb]{0.69\columnwidth}
\includegraphics[width=\columnwidth]{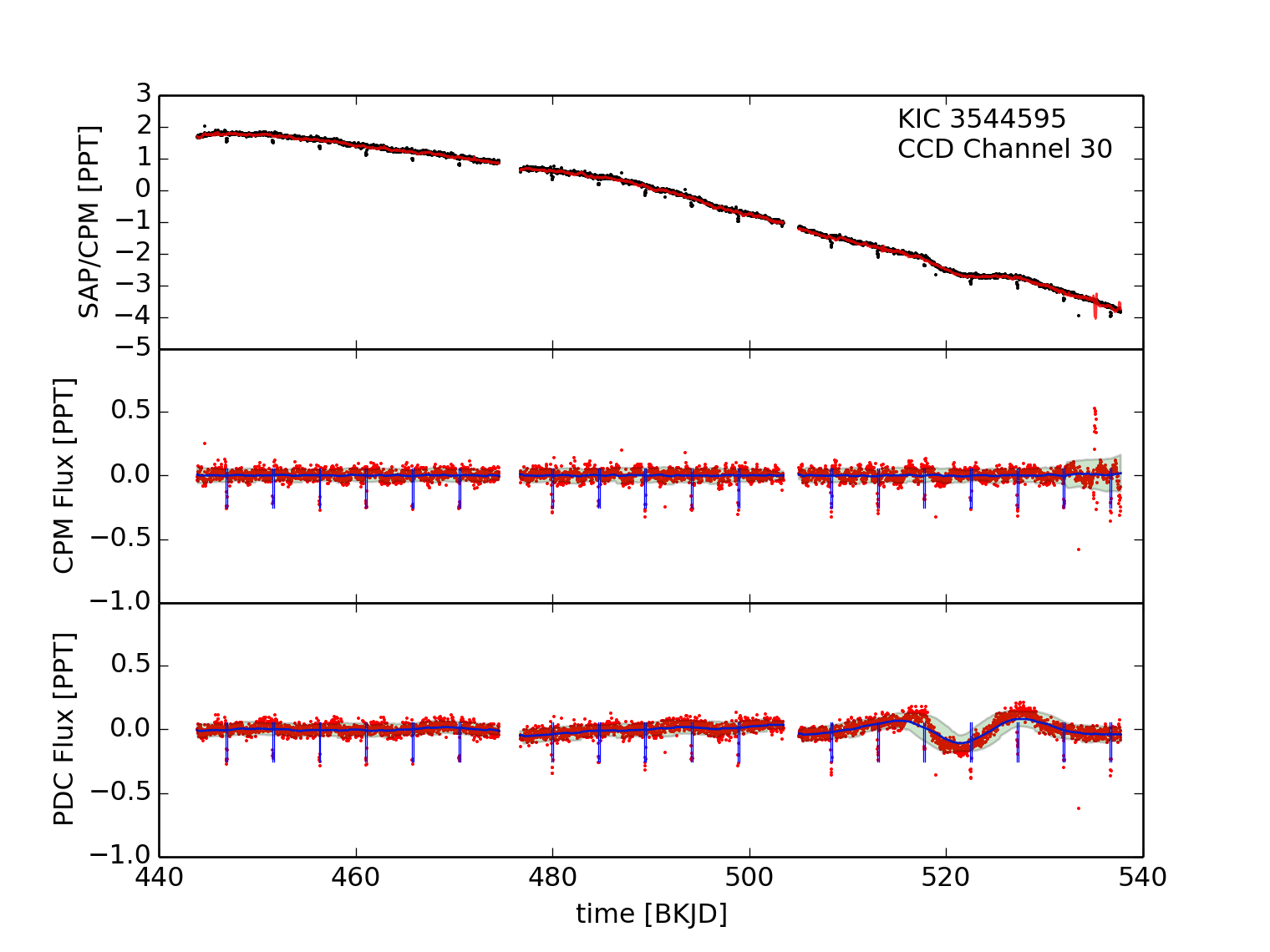}
\caption{}
\end{subfigure}%
\hfill
\begin{subfigure}[htb]{0.69\columnwidth}
\includegraphics[width=\columnwidth]{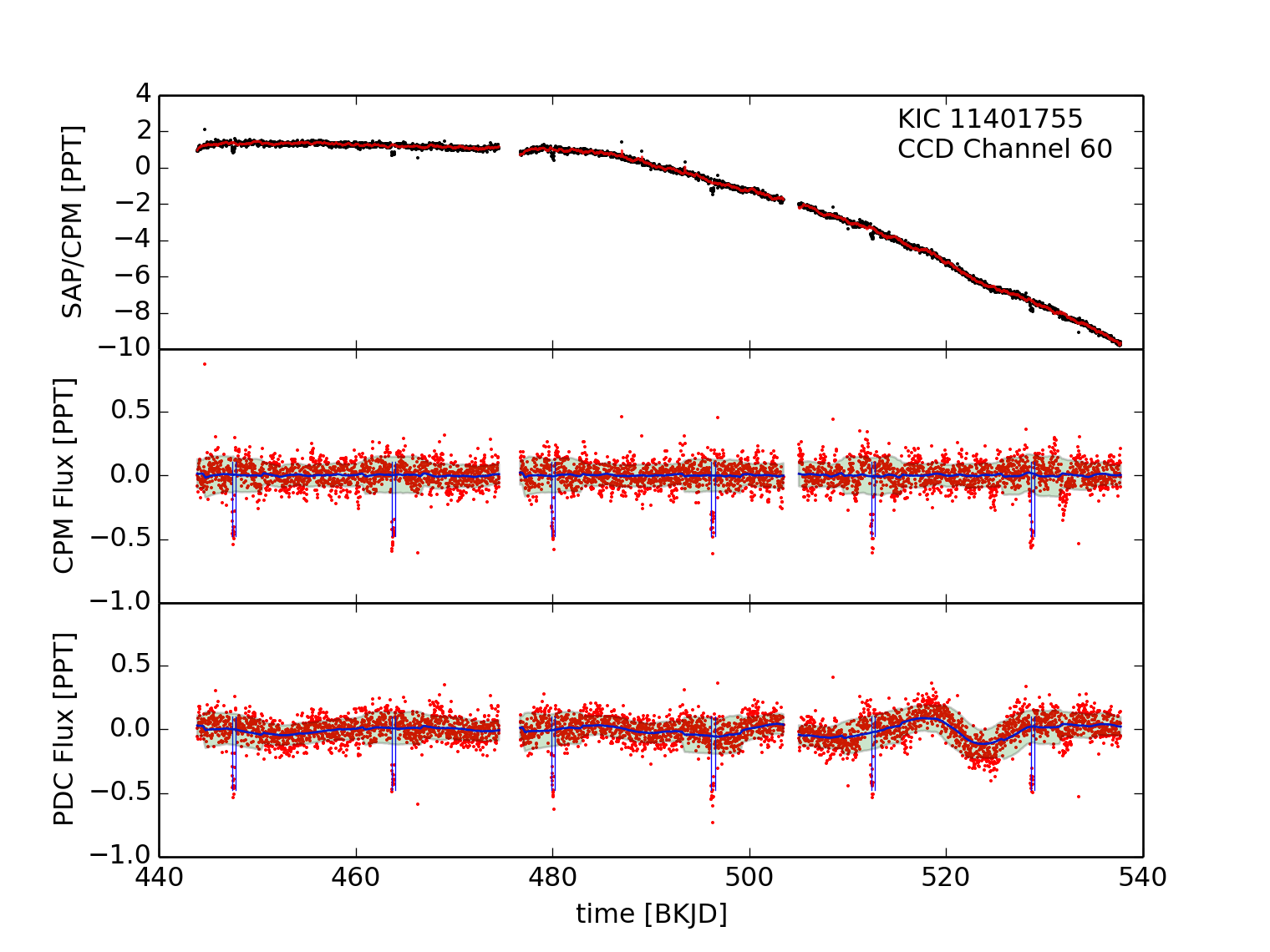}
\caption{}
\end{subfigure}%
\hfill
\begin{subfigure}[htb]{0.69\columnwidth}
\includegraphics[width=\columnwidth]{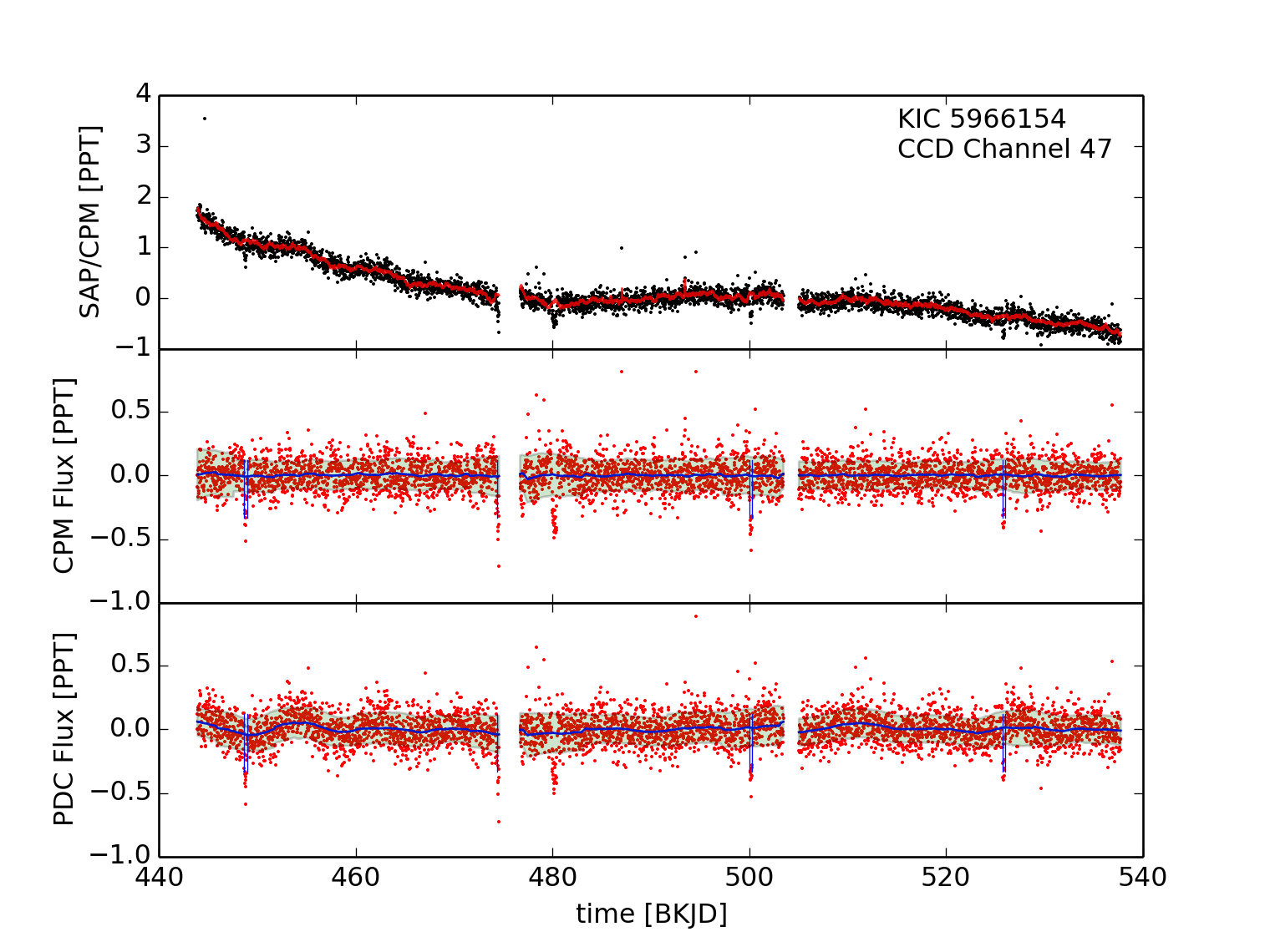}
\caption{}
\end{subfigure}
\caption{\label{fluxes}Corrected fluxes using our method, for three example stars, spanning the main magnitude (brightness) range encountered. In (a), we consider a bright star, in (b), a star of moderate brightness,  and in (c), a relatively faint star.
SAP stands for Simple Aperture Photometry (in our case, a relative flux measure computed from summing over the pixels belonging to a star).
In all three panels, the top plot shows the SAP flux (black) and the CPM regression (red), i.e., our prediction of the star from other stars. The middle panel shows the CPM flux corrected using the regression (details see text), and the bottom shows the PDC flux (i.e., the default method). The CPM flux curve preserves the exoplanet transits (little downward spikes), while removing a substantial part of the variability present in the PDC flux. All x-axes show time, measured in days since 1/1/2009.}
\end{figure*}

In Fig.~\ref{fluxes}, we present corrected light curves for three typical stars of different magnitudes, using both CPM and PDC. 
Note that in our theoretical analysis, we dealt with additive noise, and could deal with multiplicative noise, e.g., by log transforming. In practice, none of the two models is correct for our application. If we are interested in the transit (and not the stellar variability), then the variability is a multiplicative confounder. At the same time, other noises may better be modeled as additive (e.g., CCD noise). 
In practice, we calibrate the data by dividing by the regression estimate and then subtracting 1, i.e.,
$$  
\frac{Y}{E[Y|x]} - 1 
= \frac{Y}{E[Y|x]} - \frac{E[Y|x]}{E[Y|x]} 
= \frac{Y-E[Y|x]}{E[Y|x]}.$$
Effectively, we thus perform a subtractive normalization, followed by a divisive one.
This worked well, taking care of both types of contaminations.

The results illustrate that our approach removes a major part of the variability present in the PDC light curves, while preserving the transit signals. 
To provide a quantitative comparison, we ran CPM on 1000 stars from the whole Kepler input catalog (500 chosen randomly from the whole list, and 500 random G-type sun-like stars), and estimate the Combined Differential Photometric Precision (CDPP) for CPM and PDC. 
CDPP is an estimate of the relative precision in a time window, indicating the noise level seen by a transit signal with a given duration. The duration is typically chosen to be 3, 6, or 12 hours 
\cite{2012PASP..124.1279C}. 
Shorter durations are appropriate for planets close to their host stars, which are the ones that are easier to detect. We use the 12-hours CDPP metric, since the transit duration of an earth-like planet is roughly 10 hours. Fig.~\ref{cdpp} presents our CDPP comparison of CPM and PDC, showing that our method outperforms PDC. This is no small feat, since PDC is highly optimized for the task at hand, incorporating substantial astronomical knowledge (e.g., it attempts to remove stellar variability as well as systematic trends).

\begin{figure*}[t!]
\centering
\begin{subfigure}[htb]{0.7\columnwidth}
\includegraphics[width=\columnwidth]{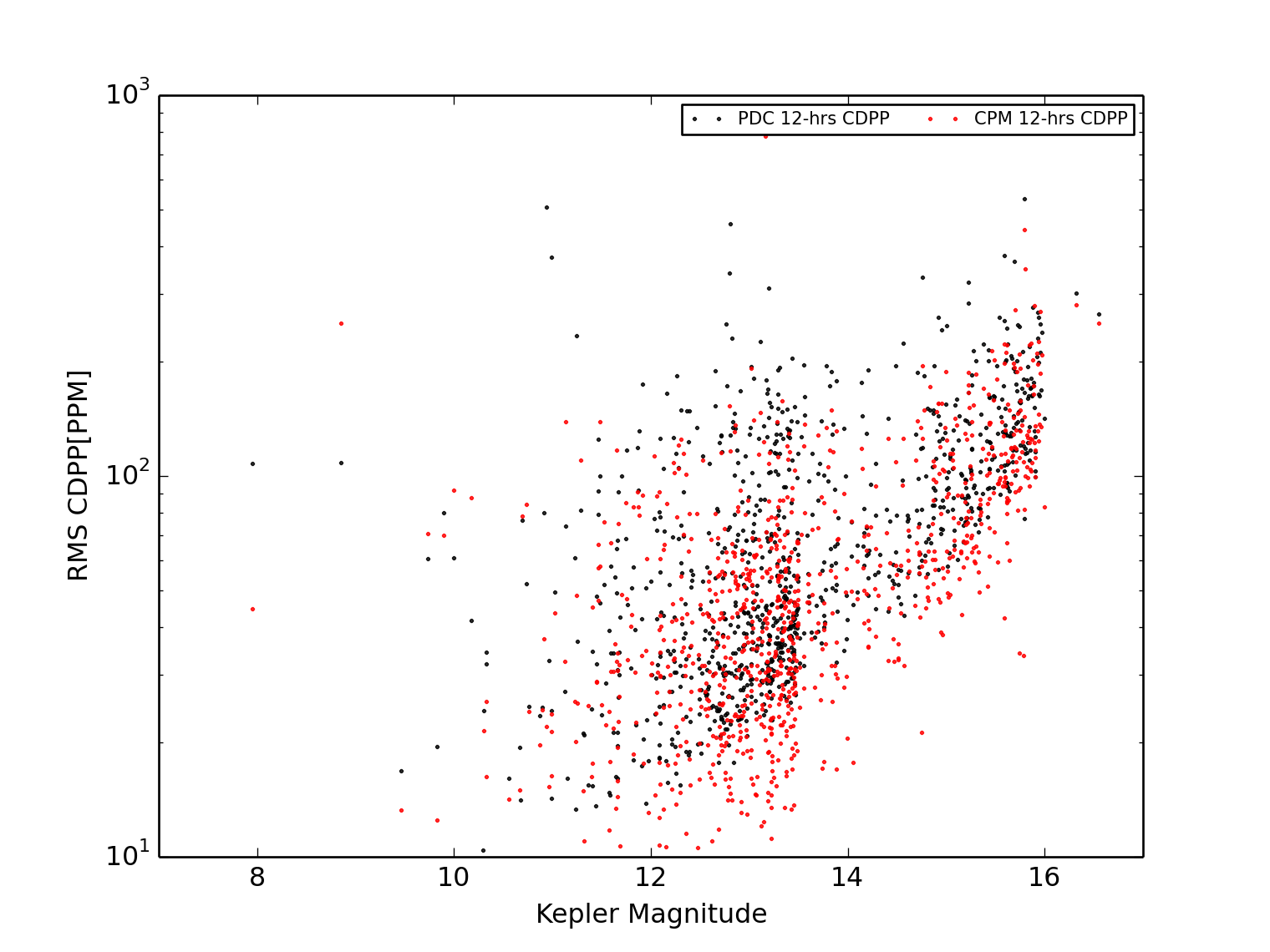}
\caption{}
\end{subfigure}%
\begin{subfigure}[htb]{0.7\columnwidth}
\includegraphics[width=\columnwidth]{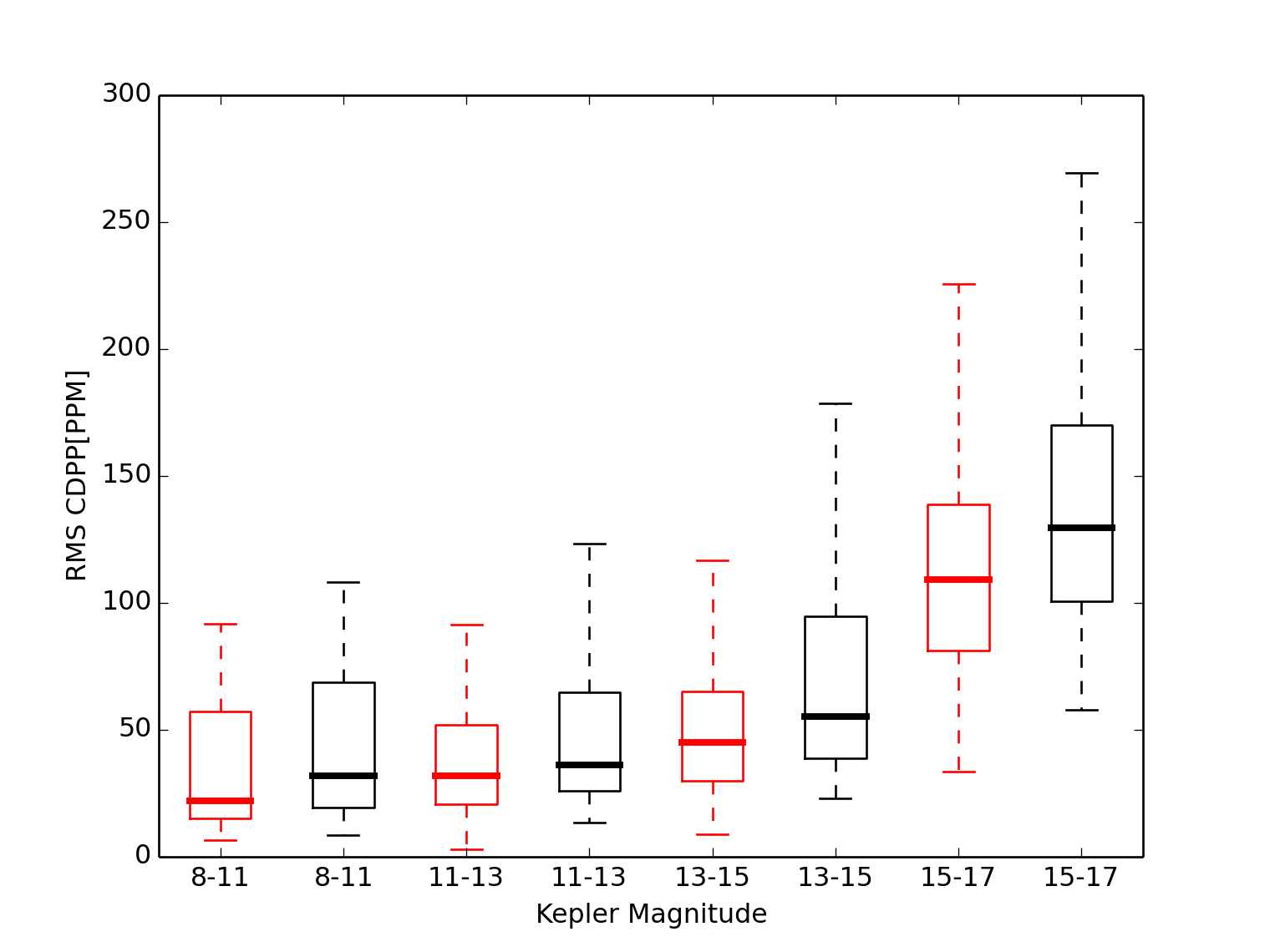}
\caption{}
\end{subfigure}%
\begin{subfigure}[htb]{0.7\columnwidth}
\includegraphics[width=\columnwidth]{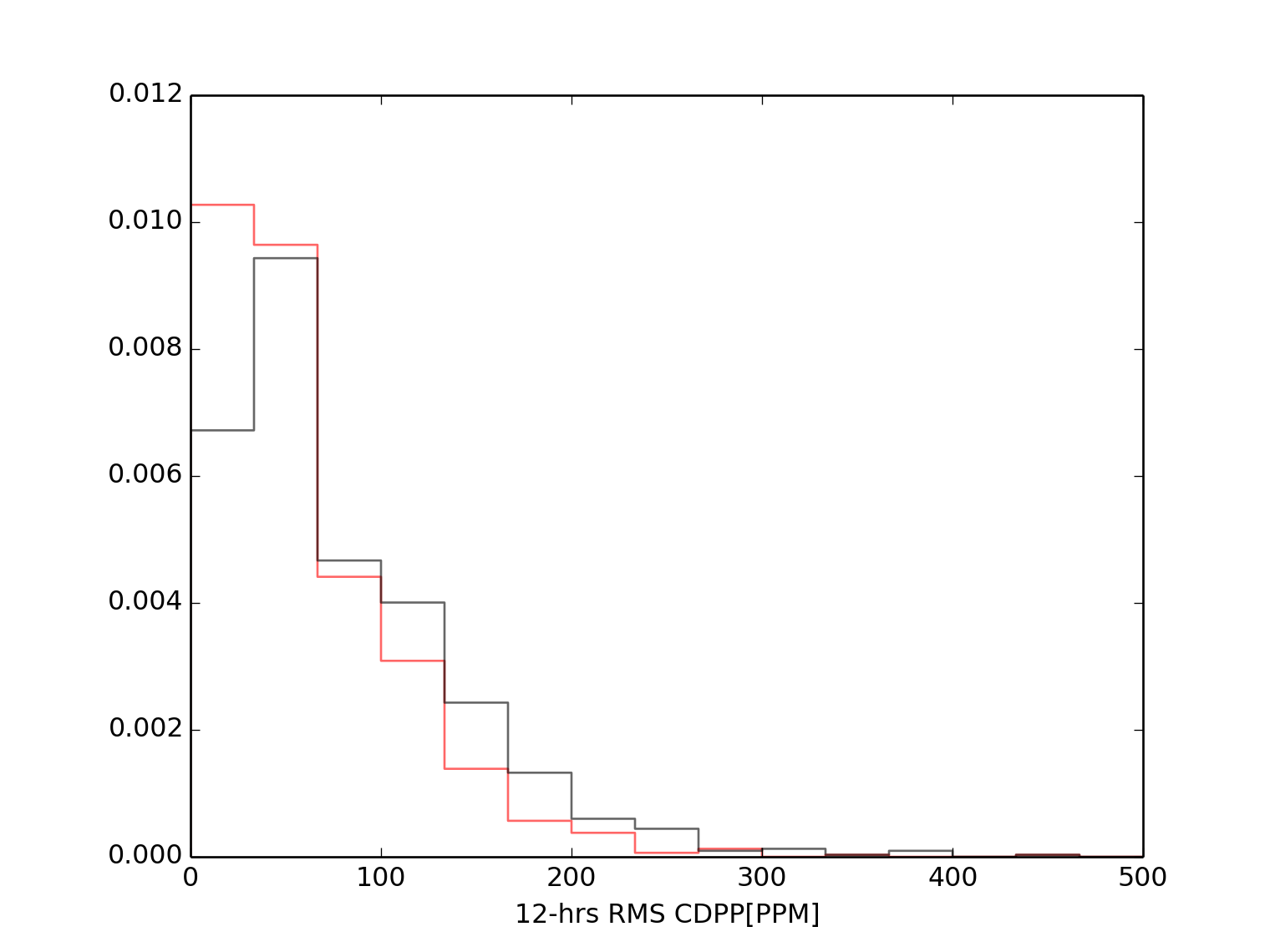}
\caption{}
\end{subfigure}
\caption{\label{cdpp}Comparison of the proposed method (CPM) to the Kepler PDC method in terms of Combined Differential Photometric Precision (CDPP) (see text). Plot (a) shows our performance (red) vs.\ the PDC performance in a scatter plot, as a function of star magnitude (note that larger magnitude means fainter stars, and smaller values of CDPP indicate a higher quality as measured by CDPP.  Plot (b) bins the same dataset and shows box plots within each bin, indicating median, top quartile and bottom quartile. The red box corresponds to CPM, while the black box refers to PDC. Plot (c), finally, shows a histogram of CDPP values. Note that the red histogram has more mass towards the left, i.e., smaller values of CDPP, indicating that our method overall outperforms PDC, the Kepler ``gold standard.'' 
}
\end{figure*}

\section{Conclusion\label{sec:conc}}
We have assayed
\iflong
{\em half-sibling regression}, 
\fi
 a simple yet effective method for removing the effect of systematic noise from observations. It utilizes the information contained in a set of other observations affected by the same noise source. 
The main motivation for the method was its application to exoplanet data processing, which we discussed in some detail, with rather promising results. However, we expect that it will have a large range of applications in other domains as well.

We expect that our method may enable astronomical discoveries at higher sensitivity on the existing Kepler satellite data. Moreover, we anticipate that methods to remove systematic errors will further increase in importance: by May 2013, two of the four reaction wheels used to control the Kepler spacecraft were disfunctional, and in May 2014, NASA announced the {\em K2} mission, using the remaining two wheels in combination with thrusters to control the spacecraft and continue the search for exoplanets in other star fields. Systematic errors in K2 data are significantly larger since the spacecraft has become harder to control. 
In addition, NASA is planning the launch of another space telescope for 2017. {\em TESS (Transiting Exoplanet Survey Satellite)}\footnote{\url{http://tess.gsfc.nasa.gov/}} will perform an all-sky survey for small (earth-like) planets of nearby stars. To date, no earth-like planets orbiting sun-like stars in the habitable zone have been found. This is likely to change in the years to come, which would be a major scientific discovery.\footnote{``Decades, or even centuries after the TESS survey is completed, the new planetary systems it discovers will continue to be studied because they are both nearby and bright. In fact, when starships transporting colonists first depart the solar system, they may well be headed toward a TESS-discovered planet as their new home.'' \cite{Haswell}} 
\ificml
Machine learning can contribute significantly towards the analysis of these datasets, and the present paper is only a start.
\fi
\iflong
In particular, while the proposed method treats the problem of removing systematic errors as a preprocessing step, we are also exploring the possibility of jointly modeling systematics and transit events. This incorporates additional knowledge about the events that are looking for in our specific application, and it has already led to promising results \cite{Foreman-Mackey15}.
\fi

\section*{Acknowledgments} We thank Stefan Harmeling, James McMurray, Oliver Stegle and
Kun Zhang for helpful discussion, and the anonymous reviewers for helpful suggestions and references. 
C-J S-G was supported
by a Google Europe Doctoral Fellowship in Causal Inference.


\vfill\eject

\bibliographystyle{icml2015}
\bibliography{bibfile2}

\end{document}

%% file: definitions.tex
\usepackage{amsmath}

\usepackage{epsf}

\usepackage{psfrag}
\usepackage{graphicx}
\usepackage{amssymb}
\usepackage{amsbsy}
\usepackage{epsfig}
\usepackage{ifthen}
\usepackage{eucal}
\usepackage{theorem}

\newtheorem{theorem}{Theorem}
\newtheorem{definition}{Definition}

\newtheorem{proposition}[theorem]{Proposition}

\newtheorem{lemma}[theorem]{Lemma}

\newenvironment{proof}[1][. ]{{\bf Proof#1}}{\hfill$\square$\vskip\baselineskip}

\newcommand{\commentout}[1]{}

\newcommand{\Var}{\mathrm{Var}}

\newcommand{\RR}{\mathbb{R}}    
    
\newcommand{\NN}{\mathbb{N}}    

\newcommand{\independent}{\perp\mkern-11mu\perp}

  \mathchardef\ordinarycolon\mathcode`\:
  \mathcode`\:=\string"8000
  \begingroup \catcode`\:=\active
    \gdef:{\mathrel{\mathop\ordinarycolon}}
  \endgroup



%% file: proposition5.tex
\begin{proposition} \label{prop:incrp}
Assume that $Y = Q + f(N)$ and that $\mathbf{X}_d := (X_1, \ldots, X_d)$ satisfies
$$
X_i := g_i(N) + R_i, \quad i=1, \ldots, d,
$$
where all $R_i$, $N$ and $Q$ are jointly independent, $\sum_{i=1}^{\infty} \frac{1}{i^2} \mathrm{var}(R_i) < \infty$, $f \in C_b^1(\mathbb{R})$, $g_i \in C^1(\mathbb{R})$ for all~$i$, and $$\tilde g_d := \frac{1}{d}\sum_{j=1}^d g_j$$ is invertible with $(\tilde g_d^{-1})_d$ uniformly equicontinuous.
Then
$$
\hat Q_d \overset{L^2}{\rightarrow} Q - E[Q] \, \quad \text{ as } \quad d \rightarrow \infty\,,
$$
where we define
$
\hat Q_d := Y - E[Y | \mathbf{X}_d]
$.
\end{proposition}

\begin{proof}
By Kolmogorov's strong law, we have for 
\iflong
$$\bar{\mu}_d := \frac{1}{d} \sum_{i=1}^d  E[R_i]$$ 
\fi
\ificml
$\bar{\mu}_d := \frac{1}{d} \sum_{i=1}^d  E[R_i]$ 
\fi
that 
\begin{eqnarray*}
& \frac{1}{d} \sum_{i=1}^d  R_i - \bar{\mu}_d \! & \! \overset{P}{\rightarrow} 0 \\
\Rightarrow  \! & \! \frac{1}{d} \sum_{i=1}^d (g_i(N) + R_i) - \bar{\mu}_d - \tilde g_d(N) &\overset{P}{\rightarrow} 0\\
\overset{*}{\Rightarrow } \! & \! \tilde g_d^{-1} \left(\frac{1}{d} \sum_{i=1}^d (g_i(N) + R_i) - \bar{\mu}_d\right) - \tilde g_d^{-1} \left( \tilde g_d(N) \right)\! & \! \overset{P}{\rightarrow} 0\\
\Rightarrow  \! & \! \tilde g_d^{-1} \left(\frac{1}{d} \sum_{i=1}^d X_i - \bar{\mu}_d \right) - N \! & \! \overset{P}{\rightarrow} 0\\
\overset{**}{\Rightarrow } \! & \! f\left(\tilde g_d^{-1} \left(\frac{1}{d} \sum_{i=1}^d X_i - \bar{\mu}_d \right) \right) - f(N) \! & \! \overset{P}{\rightarrow} 0\\
\Rightarrow  \! & \! \psi_d (\mathbf{X}_d) - f(N) \! & \! \overset{P}{\rightarrow} 0
\end{eqnarray*}
for some $\psi_d$ that are uniformly bounded in $d$ (the implication $*$ follows from uniform equicontinuity, implication~$**$ by the continuous mapping theorem).\ificml
\footnote{The notation $\overset{P}{\rightarrow}$ denotes convergence in probability with respect to the measure $P$ of the underlying probability space.}
\fi
~This implies
$$
E[f(N) | \mathbf{X}_d] - f(N)  \overset{L^2}{\rightarrow} 0
$$
because 
$$
E[\left(f(N) - E[f(N)|\mathbf{X}_d]\right)^2] \leq E[\left(f(N) - \psi_d(\mathbf{X}_d) \right)^2] {\rightarrow} 0
$$
(The convergence of the right hand side follows from $\psi_d (\mathbf{X}_d) - f(N)  \overset{P}{\rightarrow} 0$ and boundedness of $\psi_d (\mathbf{X}_d) - f(N)$).
But then
\begin{align*}
Q - E[Q] - \hat Q_d 
&= - f(N) -E[Q] + E[f(N) + Q | \mathbf{X}_d] \\
&= E[f(N) | \mathbf{X}_d] - f(N)  \overset{L^2}{\rightarrow} 0
\end{align*}

\ificml
\vglue-6mm
\fi
\end{proof}

\ificml
\vglue-3mm
\fi

%% file: timeseries.tex
Above, we have worked with random variables and assumed that the regression is performed on i.i.d.\ data drawn from those random variables. However, in practice we also encounter problems where the data are drawn from random processes depending on time.

Consider a causal graph with an additional confounder $T$ representing time, see Figure~\ref{fig:DAG-t}, and assume that the signals $R$ and $Q$ have a time series structure. 
This representation becomes necessary if $R$ and $Q$ share a strong periodicity, for example. If we want to retain this periodicity, we should not simply regress $Y$ on $X$.
\begin{figure}[bt]
\begin{center}
\begin{tikzpicture}[xscale = 2, yscale=1.7, line width=0.5pt, inner sep=0.5mm, shorten >=1pt, shorten <=1pt]
  \draw (-.5,2) node(oo) [] {unobserved};
  \draw (-.5,1) node(o) [] {observed};
  \draw (1,1) node(y) [circle, draw] {$Y$};
  \draw (2.6,1) node(x) [circle, draw] {$X$};
  \draw (2.6,2) node(r) [circle, draw] {$R$};
  \draw (1.8,2) node(n) [circle, draw] {$N$};
  \draw (1.8,1) node(t) [circle, draw] {$T$};
  \draw (1,2) node(q) [circle, draw] {$Q$};
  \draw[-arcsq] (q) -- (y);
  \draw[-arcsq] (n) -- (y);
  \draw[-arcsq] (n) -- (x);
  \draw[-arcsq] (t) -- (q);
  \draw[-arcsq] (r) -- (x);
  \draw[-arcsq] (t) -- (r);
\end{tikzpicture}
\end{center}
\caption
{While Fig.~\ref{fig:DAG} refers to i.i.d.\ data, the present figure includes an effect of time $T$ on our quantity of interest, $Q$, and through the signal $R$ on our predictors $X$ which are affected by the same noise $N$. 
Simply regressing $Y$ on $X$ as in the i.i.d.\ case removes some of the signal $Q$ from $Y$. 
Allowing for an edge $T \rightarrow N$ makes the problem even more difficult.
\label{fig:DAG-t}}
\end{figure}

In many applications the signals may have a time structure but we expect $R$ and $Q$ as well as $Q$ and $N$ to be independent. We further assume that the signals $R$ and $Q$ will normally not share any strong frequencies.
In those situations the representation shown in Figure~\ref{fig:DAG-t2} may be more appropriate.
\begin{figure}[bt]
\begin{center}
\begin{tikzpicture}[xscale = 1.5, yscale=1.4, line width=0.5pt, inner sep=0.3mm, shorten >=1pt, shorten <=1pt]
\small
  \draw (1,3) node(qt1) [circle, draw] {$Q_{t-1}$};
  \draw (2,3) node(yt1) [circle, draw] {$Y_{t-1}$};
  \draw (3,3) node(nt1) [circle, draw] {$N_{t-1}$};
  \draw (4,3) node(xt1) [circle, draw] {$X_{t-1}$};
  \draw (5,3) node(rt1) [circle, draw] {$R_{t-1}$};
  \draw (1,2) node(qt2) [circle, draw] {$\;\;Q_{t}\;\;$};
  \draw (2,2) node(yt2) [circle, draw] {$\;\;Y_{t}\;\;$};
  \draw (3,2) node(nt2) [circle, draw] {$\;\;N_{t}\;\;$};
  \draw (4,2) node(xt2) [circle, draw] {$\;\;X_{t}\;\;$};
  \draw (5,2) node(rt2) [circle, draw] {$\;\;R_{t}\;\;$};
  \draw (1,1) node(qt3) [circle, draw] {$Q_{t+1}$};
  \draw (2,1) node(yt3) [circle, draw] {$Y_{t+1}$};
  \draw (3,1) node(nt3) [circle, draw] {$N_{t+1}$};
  \draw (4,1) node(xt3) [circle, draw] {$X_{t+1}$};
  \draw (5,1) node(rt3) [circle, draw] {$R_{t+1}$};
  \draw[-arcsq] (qt1) -- (qt2);
  \draw[-arcsq] (qt2) -- (qt3);
  \draw[-arcsq] (rt1) -- (rt2);
  \draw[-arcsq] (rt2) -- (rt3);
  \draw[-arcsq] (qt1) -- (yt1);
  \draw[-arcsq] (qt2) -- (yt2);
  \draw[-arcsq] (qt3) -- (yt3);
  \draw[-arcsq] (rt1) -- (xt1);
  \draw[-arcsq] (rt2) -- (xt2);
  \draw[-arcsq] (rt3) -- (xt3);
  \draw[-arcsq] (nt1) -- (yt1);
  \draw[-arcsq] (nt2) -- (yt2);
  \draw[-arcsq] (nt3) -- (yt3);
  \draw[-arcsq] (nt1) -- (xt1);
  \draw[-arcsq] (nt2) -- (xt2);
  \draw[-arcsq] (nt3) -- (xt3);
\end{tikzpicture}
\end{center}
\caption
{Special case of Figure~\ref{fig:DAG-t}. Here, the signals $Q$ and $R$ are independent and thus regressing $Y_t$ on $X_t$ is valid in the sense that it would not remove any information of $Q_t$ from $Y_t$.
\label{fig:DAG-t2}}
\end{figure}
Because of the independence between $Q$ and $R$, we can proceed as before and estimate $Q_t$ as the residuals after regressing $Y_t$ from $X_t$ (we could even allow $N$ to have a time structure, too). The graph structure shows that after including $X_t$ as a predictor for $Y_t$, all other $X_{t+h}, h \neq 0$ may contain further information about $Y_t$. Note however, that this dependence decreases quickly with increasing $|h|$, especially when the contribution of $R_t$ to $X_t$ is small compared to the contribution of $N_t$ to $X_t$. 
Still, in some simulation settings, including different time lags $X_{t+h}, h \in \{\ldots,-1,0,1,\ldots\}$ into the model for $Y_t$ improves the performance of the method (in terms of reconstructing $Q$) compared to predicting $Y_t$ only from $X_t$ (results not shown).
We expect that identifiability statements similar to the i.i.d.\ case may hold (see sections~\ref{sec:complete} and~\ref{sec:incompl}).

%% file: experiments1.tex
We analyze two simulated data sets that illustrate the identifiability statements from Sections~\ref{sec:complete} and~\ref{sec:incompl}.

\paragraph{Increasing relative strength of $N$ in a single $X$.}
We consider $20$ instances (each time we sample $200$ i.i.d.\ data points) of the model $Y = f(N) + Q$ and $X = g(N) + R$, where $f$ and $g$ are randomly chosen sigmoid functions and the variables $N$, $Q$ and $R$ are normally distributed. The standard deviation for $R$ is chosen uniformly between $0.05$ and~$1$, the standard deviation for $N$ is between $0.5$ and~$1$. Because $Q$ can be recovered only up to a shift in the mean, we set its sample mean to zero. The distribution for $R$, however, has a mean that is chosen uniformly between $-1$ and~$1$ and its standard deviation is chosen from the vector $(1, 0.5, 0.25, 0.125, 0.0625,0.03125,0)$.
Proposition~\ref{prop:decrsd} shows that with decreasing standard deviation of $R$ we can recover the signal $Q$. 
Standard deviation zero corresponds to the case of complete information (Section~\ref{sec:complete}). 
For regressing $Y$ on $X$, we use the function \texttt{gam} (penalized regression splines) from the \texttt{R}-package \texttt{mgcv}; 
Figure~\ref{fig:simJonas} shows that this asymptotic behavior can be seen on finite data sets. 

\paragraph{Increasing number of observed $X_i$ variables.}
Here, we consider the same simulation setting as before, this time simulating
$X_i = g_i(N) + R_i$ for $i = 1, \ldots, p$. 
We have shown in Proposition~\ref{prop:incrp} that if the number of variables $X_i$ tends to infinity, we are able to reconstruct the signal $Q$.
In this experiment, the standard deviation for $R_i$ and $Q$ is chosen uniformly between $0.05$ and $1$; The distribution of $N$ is the same as above.
It is interesting to note that even additive models (in the predictor variables) work as a regression method (we use the function \texttt{gam} from the \texttt{R}-package \texttt{mgcv} on all variables $X_1, \ldots, X_p$ and its sum $X_1 + \ldots + X_p$).
Figure~\ref{fig:simJonas} shows that with increasing $p$ the reconstruction of $Q$ improves.
\begin{figure}
\begin{center}
\includegraphics[width=0.49 \columnwidth]{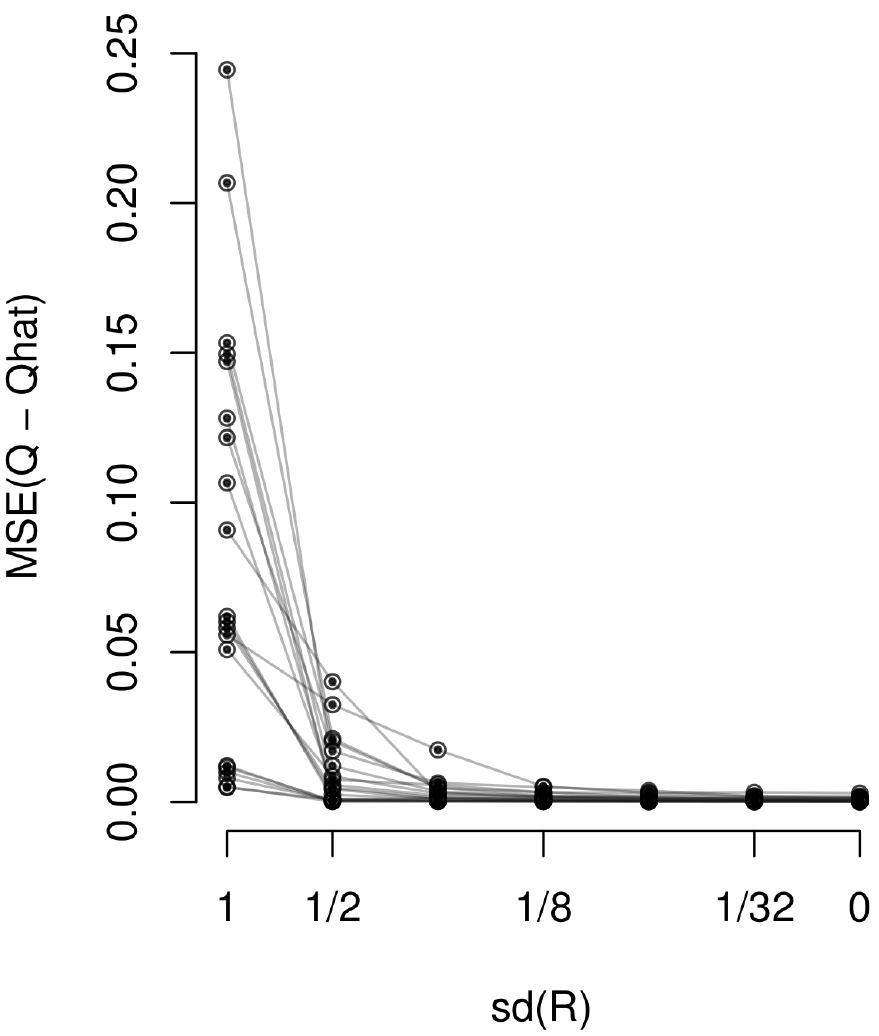}
\hfill
\includegraphics[width=0.49 \columnwidth]{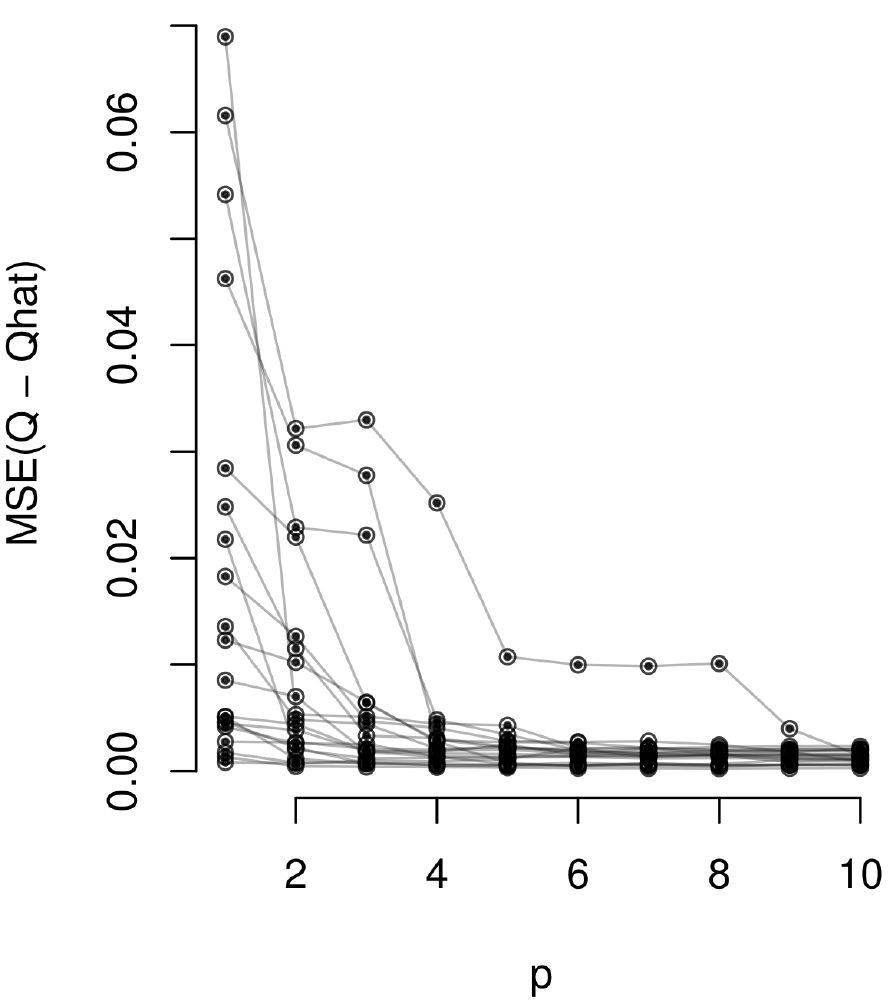}
\end{center}
\caption{
Left: we observe a variable $X = g(N) + R$ with invertible function $g$. If the variance of $R$ decreases, the reconstruction of $Q$ improves because it becomes easier to remove the influence $f(N)$ of the noise $N$ from the variable $Y = f(N) + Q$ by using $X$, see Proposition~\ref{prop:decrsd}.
Right:
a similar behavior occurs with increasing the number $p$ of predictor variables $X_i = g_i(N) + R_i$, see Proposition~\ref{prop:incrp}. Both plots show $20$ scenarios, each connected by a thin line.}
\label{fig:simJonas}
\end{figure}

%% file: paper-ICML-new.bbl
\begin{thebibliography}{18}
\providecommand{\natexlab}[1]{#1}
\providecommand{\url}[1]{\texttt{#1}}
\expandafter\ifx\csname urlstyle\endcsname\relax
  \providecommand{\doi}[1]{doi: #1}\else
  \providecommand{\doi}{doi: \begingroup \urlstyle{rm}\Url}\fi

\bibitem[{Christiansen} et~al.(2012){Christiansen}, {Jenkins}, {Caldwell},
  {Burke}, {Tenenbaum}, {Seader}, {Thompson}, {Barclay}, {Clarke}, {Li},
  {Smith}, {Stumpe}, {Twicken}, and {Van Cleve}]{2012PASP..124.1279C}
{Christiansen}, J.~L., {Jenkins}, J.~M., {Caldwell}, D.~A., {Burke}, C.~J.,
  {Tenenbaum}, P., {Seader}, S., {Thompson}, S.~E., {Barclay}, T.~S., {Clarke},
  B.~D., {Li}, J., {Smith}, J.~C., {Stumpe}, M.~C., {Twicken}, J.~D., and {Van
  Cleve}, J.
\newblock {The Derivation, Properties, and Value of Kepler's Combined
  Differential Photometric Precision}.
\newblock \emph{Publications of the Astronomical Society of the Pacific},
  124:\penalty0 1279--1287, 2012.

\bibitem[Foreman-Mackey et~al.(2015)Foreman-Mackey, Montet, Hogg, Morton, Wang,
  and Sch{\"o}lkopf]{Foreman-Mackey15}
Foreman-Mackey, D., Montet, B.~T., Hogg, D.~W., Morton, T.~D., Wang, D., and
  Sch{\"o}lkopf, B.
\newblock A systematic search for transiting planets in the {K2} data.
\newblock \emph{arXiv:1502.04715}, 2015.

\bibitem[Gagnon-Bartsch \& Speed(2011)Gagnon-Bartsch and Speed]{GagSpe11}
Gagnon-Bartsch, J.~A. and Speed, T.~P.
\newblock \emph{Biostatistics}, 13:\penalty0 539--552, 2011.

\bibitem[Haswell(2010)]{Haswell}
Haswell, Carole~A.
\newblock \emph{Transiting Exoplanets}.
\newblock Cambridge University Press, 2010.

\bibitem[Hoyer et~al.(2009)Hoyer, Janzing, Mooij, Peters, and
  Sch{\"o}lkopf]{HoyJanMooPetetal09}
Hoyer, P.~O., Janzing, D., Mooij, J.~M., Peters, J., and Sch{\"o}lkopf, B.
\newblock Nonlinear causal discovery with additive noise models.
\newblock In Koller, D., Schuurmans, D., Bengio, Y., and Bottou, L. (eds.),
  \emph{Advances in Neural Information Processing Systems}, volume~21, pp.\
  689--696, 2009.

\bibitem[Janzing et~al.(2009)Janzing, Peters, Mooij, and
  Sch{\"o}lkopf]{JanPetMooSch09}
Janzing, D., Peters, J., Mooij, J., and Sch{\"o}lkopf, B.
\newblock Identifying confounders using additive noise models.
\newblock In Bilmes, J and Ng, AY (eds.), \emph{25th Conference on Uncertainty
  in Artificial Intelligence}, pp.\  249--257, Corvallis, OR, USA, 2009. AUAI
  Press.

\bibitem[Johnson \& Li(2007)Johnson and Li]{JohLi07}
Johnson, W.~E. and Li, C.
\newblock Adjusting batch effects in microarray expression data using empirical
  {B}ayes methods.
\newblock \emph{Biostatistics}, 8:\penalty0 118–127, 2007.

\bibitem[Kang et~al.(2008)Kang, Ye, and Eskin]{Kang08}
Kang, H.~M., Ye, C., and Eskin, E.
\newblock \emph{Genetics}, 180\penalty0 (4):\penalty0 1909–--1925, 2008.

\bibitem[{Padmanabhan} et~al.(2008){Padmanabhan}, {Schlegel}, {Finkbeiner},
  {Barentine}, {Blanton}, {Brewington}, {Gunn}, {Harvanek}, {Hogg},
  {Ivezi{\'c}}, {Johnston}, {Kent}, {Kleinman}, {Knapp}, {Krzesinski}, {Long},
  {Neilsen}, {Nitta}, {Loomis}, {Lupton}, {Roweis}, {Snedden}, {Strauss}, and
  {Tucker}]{self-calib}
{Padmanabhan}, N., {Schlegel}, D.~J., {Finkbeiner}, D.~P., {Barentine}, J.~C.,
  {Blanton}, M.~R., {Brewington}, H.~J., {Gunn}, J.~E., {Harvanek}, M., {Hogg},
  D.~W., {Ivezi{\'c}}, {\v Z}., {Johnston}, D., {Kent}, S.~M., {Kleinman},
  S.~J., {Knapp}, G.~R., {Krzesinski}, J., {Long}, D., {Neilsen}, Jr., E.~H.,
  {Nitta}, A., {Loomis}, C., {Lupton}, R.~H., {Roweis}, S., {Snedden}, S.~A.,
  {Strauss}, M.~A., and {Tucker}, D.~L.
\newblock {An Improved Photometric Calibration of the Sloan Digital Sky Survey
  Imaging Data}.
\newblock \emph{The Astrophysical Journal}, 674:\penalty0 1217--1233, 2008.
\newblock \doi{10.1086/524677}.

\bibitem[Pearl(2000)]{Pearl00}
Pearl, J.
\newblock \emph{Causality}.
\newblock Cambridge University Press, 2000.

\bibitem[Peters et~al.(2014)Peters, Mooij, Janzing, and
  Sch\"olkopf]{Peters2014anm}
Peters, J., Mooij, J.M., Janzing, D., and Sch\"olkopf, B.
\newblock Causal discovery with continuous additive noise models.
\newblock \emph{Journal of Machine Learning Research}, 15:\penalty0 2009--2053,
  2014.

\bibitem[Price et~al.(2006)Price, Patterson, Plenge, Weinblatt, Shadick, and
  Reich]{Price06}
Price, Alkes~L, Patterson, Nick~J, Plenge, Robert~M, Weinblatt, Michael~E,
  Shadick, Nancy~A, and Reich, David.
\newblock \emph{Nature Genetics}, 38\penalty0 (8):\penalty0 904--909, 2006.

\bibitem[Sch{\"o}lkopf et~al.(2012)Sch{\"o}lkopf, Janzing, Peters, Sgouritsa,
  Zhang, and Mooij]{ScholkopfJPSZMJ2012}
Sch{\"o}lkopf, B., Janzing, D., Peters, J., Sgouritsa, E., Zhang, K., and
  Mooij, J.~M.
\newblock On causal and anticausal learning.
\newblock In Langford, J and Pineau, J (eds.), \emph{Proceedings of the 29th
  International Conference on Machine Learning ({ICML})}, pp.\  1255--1262, New
  York, NY, USA, 2012. Omnipress.

\bibitem[{Smith} et~al.(2012){Smith}, {Stumpe}, {Van Cleve}, {Jenkins},
  {Barclay}, {Fanelli}, {Girouard}, {Kolodziejczak}, {McCauliff}, {Morris}, and
  {Twicken}]{pdc3}
{Smith}, J.~C., {Stumpe}, M.~C., {Van Cleve}, J.~E., {Jenkins}, J.~M.,
  {Barclay}, T.~S., {Fanelli}, M.~N., {Girouard}, F.~R., {Kolodziejczak},
  J.~J., {McCauliff}, S.~D., {Morris}, R.~L., and {Twicken}, J.~D.
\newblock {Kepler Presearch Data Conditioning II - A Bayesian Approach to
  Systematic Error Correction}.
\newblock \emph{Publications of the Astronomical Society of the Pacific},
  124:\penalty0 1000--1014, September 2012.
\newblock \doi{10.1086/667697}.

\bibitem[Spirtes et~al.(1993)Spirtes, Glymour, and Scheines]{SpiGlySch93}
Spirtes, P., Glymour, C., and Scheines, R.
\newblock \emph{Causation, prediction, and search}.
\newblock Springer-Verlag. (2nd edition MIT Press 2000), 1993.

\bibitem[Stegle et~al.(2008)Stegle, Kannan, Durbin, and Winn]{Stegle08}
Stegle, Oliver, Kannan, Anitha, Durbin, Richard, and Winn, John~M.
\newblock Accounting for non-genetic factors improves the power of {eQTL}
  studies.
\newblock In \emph{Proc. Research in Computational Molecular Biology, 12th
  Annual International Conference, {RECOMB}}, pp.\  411--422, 2008.

\bibitem[{Stumpe} et~al.(2012){Stumpe}, {Smith}, {Van Cleve}, {Twicken},
  {Barclay}, {Fanelli}, {Girouard}, {Jenkins}, {Kolodziejczak}, {McCauliff},
  and {Morris}]{pdc2}
{Stumpe}, M.~C., {Smith}, J.~C., {Van Cleve}, J.~E., {Twicken}, J.~D.,
  {Barclay}, T.~S., {Fanelli}, M.~N., {Girouard}, F.~R., {Jenkins}, J.~M.,
  {Kolodziejczak}, J.~J., {McCauliff}, S.~D., and {Morris}, R.~L.
\newblock {Kepler Presearch Data Conditioning I - Architecture and Algorithms
  for Error Correction in Kepler Light Curves}.
\newblock \emph{Publications of the Astronomical Society of the Pacific},
  124:\penalty0 985--999, 2012.

\bibitem[Yu et~al.(2006)Yu, Pressoir, Briggs, Vroh~Bi, Yamasaki, Doebley,
  McMullen, Gaut, Nielsen, Holland, Kresovich, and Buckler]{Yu06}
Yu, Jianming, Pressoir, Gael, Briggs, William~H, Vroh~Bi, Irie, Yamasaki,
  Masanori, Doebley, John~F, McMullen, Michael~D, Gaut, Brandon~S, Nielsen,
  Dahlia~M, Holland, James~B, Kresovich, Stephen, and Buckler, Edward~S.
\newblock A unified mixed-model method for association mapping that accounts
  for multiple levels of relatedness.
\newblock \emph{Nature Genetics}, 38\penalty0 (2):\penalty0 203--208, 2006.

\end{thebibliography}
